%% file: main.tex
\RequirePackage{fix-cm}
\documentclass[twocolumn]{svjour3}          
\smartqed  
\usepackage{xcolor}
\usepackage{tabularx}
\usepackage{multirow}
\usepackage{graphicx}
\usepackage{makecell}
\usepackage{amsmath}
\usepackage{color}
\usepackage{float}
\usepackage{rotating}
\usepackage{placeins}
\usepackage{amsfonts}
\usepackage{bm}

\def\eg{\textit{e.g.}}
\def\ie{\textit{i.e.}}

\usepackage[title]{appendix}

\usepackage[round,authoryear,comma,sort&compress,numbers]{natbib} 
\usepackage[colorlinks=true,linkcolor=blue,citecolor=black,urlcolor=blue]{hyperref}

\usepackage{diagbox}
\usepackage{cancel} 
\usepackage{mathrsfs} 
\usepackage{eqnarray} 

\usepackage{ctable} 

\usepackage{algorithm}
\usepackage{algpseudocode}%
\usepackage{listings}%
\usepackage{bbding}
\usepackage{bbm}
\usepackage[caption=false, font=footnotesize]{subfig}
\usepackage{fontawesome}
\usepackage{array}

\usepackage{colortbl}  
\usepackage{arydshln}
\usepackage{longtable}
\usepackage{wrapfig}
\usepackage{bbold}
\usepackage{caption}
\usepackage{pifont}

\newcommand{\PreserveBackslash}[1]{\let\temp=\\#1\let\\=\temp}
\newcolumntype{C}[1]{>{\PreserveBackslash\centering}p{#1}}
\newcolumntype{R}[1]{>{\PreserveBackslash\raggedleft}p{#1}}
\newcolumntype{L}[1]{>{\PreserveBackslash\raggedright}p{#1}}

\def\eg{\textit{e.g.}}
\def\ie{\textit{i.e.}}

\definecolor{battleshipgrey}{rgb}{0.52, 0.52, 0.51}
\definecolor{capri}{rgb}{0.0, 0.75, 1.0}
\definecolor{mediumspringgreen}{rgb}{0.0, 0.98, 0.6}
\definecolor{Gray}{rgb}{0.7,0.7,0.7}
\newcommand{\fullmodelname}{preconditioned diffusion sampling}
\newcommand{\shortmodelname}{PDS}

\journalname{International Journal of Computer Vision}
\begin{document}

\title{Preconditioned Score-based Generative Models
}


\author{
	Hengyuan Ma$^1$ \and
	Xiatian Zhu$^2$ \and
        Jianfeng Feng$^1$ \and
        Li Zhang$^1$
}




\institute{
	Corresponding author: Li Zhang  \at
             \email{lizhangfd@fudan.edu.cn}          \\
$^1$Fudan University, Shanghai, China \\
$^2$University of Surrey, Guildford, UK \\
H. Ma and J. Feng are with Institute of Science and Technology for Brain-inspired Intelligence, Fudan University.\\
L. Zhang is with School of Data Science, Fudan University.
}
\date{27th Feb 2025}

\maketitle

\input{sections/0_abstract.tex}
\input{sections/1_introduction.tex}
\input{sections/2_related_work.tex}
\input{sections/3_preliminaries}
\input{sections/4_theory}
\input{sections/5_pds}
\input{sections/6_further}
\input{sections/7_exp}
\input{sections/8_further2}
\input{sections/9_conclusion}

\input{Appendix}

\section*{Data availability}
The datasets used in the study are available in CIFAR-10~\cite{krizhevsky2009learning}, ImageNet~\cite{deng2009imagenet}, COCO~\cite{lin2014microsoft}, LSUN~\cite{yu2015lsun}, CelebA~\cite{liu2015faceattributes}, and FFHQ~\cite{karras2019style}.

\section*{Acknowledgments}
This work was supported in part by 
Science and Technology Innovation 2030 - Brain Science and Brain-Inspired Intelligence Project (Grant No. 2021ZD0200204) and National Natural Science Foundation of China (Grant No. 62376060).

\clearpage
\bibliographystyle{spbasic}      
\bibliography{main}   


\end{document}

%% file: sections/0_abstract.tex
\begin{abstract}
Score-based generative models (SGMs) have recently emerged as a promising class of generative models.
However, a fundamental limitation is that their sampling process is slow due to a need for many (\eg, $2000$) iterations of sequential computations.
An intuitive acceleration method is to reduce the sampling iterations which however causes severe performance degradation.
We assault this problem to the ill-conditioned issues of the Langevin dynamics and reverse diffusion in the sampling process.
Under this insight, we propose a novel {\bf\em preconditioned diffusion sampling} (PDS) method that leverages matrix preconditioning to alleviate the aforementioned problem.
PDS alters the sampling process of a vanilla SGM at marginal extra computation cost and without model retraining.
Theoretically, we prove that PDS preserves the output distribution of the SGM, with no risk of inducing systematical bias to the original sampling process.
We further theoretically reveal a relation between the parameter of PDS and the sampling iterations, 
easing the parameter estimation under varying sampling iterations.
Extensive experiments on various image datasets with a variety of resolutions and diversity validate that our PDS consistently accelerates off-the-shelf SGMs whilst maintaining the synthesis quality.
In particular, PDS can accelerate by up to $28\times$ on more challenging high-resolution (1024$\times$1024) image generation. 
Compared with the latest generative models (\eg, CLD-SGM and Analytic-DDIM), PDS can achieve the best sampling quality on CIFAR-10 at an FID score of 1.99.
Our code is publicly available to foster any further research \url{https://github.com/fudan-zvg/PDS}.
\keywords{Image synthesis, Score-based generative models, Preconditioning}
\end{abstract}

%% file: sections/1_introduction.tex
\section{Introduction}
As an alternative approach to generative adversarial networks (GANs)~\citep{goodfellow2014generative}, recent score-based generative models (SGMs)~\citep{yang2023diffusion}
have demonstrated excellent abilities in data synthesis (especially in high-resolution images) with easier optimization \citep{song2019generative}, richer diversity \citep{xiao2021tackling}, and more solid {theoretical} foundation \citep{de2021diffusion}. 

Based on these progresses, SGMs and related models show great potential in various application\citep{yang2022diffusion,yang2023diffusion} including high fidelity audio streams generation \citep{DBLP:conf/iclr/KongPHZC21,DBLP:conf/iclr/ChenZZWNC21}, text-to-image generation \citep{nichol2021glide,saharia2022photorealistic,rombach2022high}, photo-realistic images editing \citep{meng2021sdedit,saharia2021palette,saharia2021image,kawar2022denoising}, video~\citep{ho2022video} generation, 3D shape generation~\citep{zhou20213d}, and scene graphs generation~\citep{yang2022diffusion}.
Starting from a sample initialized with a Gaussian distribution, {an SGM} produces a target sample by simulating a sampling process, typically a reverse diffusion, Langevin dynamics, or their combination (Alg.~\ref{alg:original_sample}).
Compared to state-of-the-art GANs \citep{DBLP:conf/iclr/BrockDS19,karras2019style,DBLP:conf/iclr/KarrasALL18}, a significant drawback {of} existing SGMs is {\em drastically slower generation} due to the need to take many iterations for a sequential sampling process \citep{song2020score,nichol2021glide,xiao2021tackling}. 

Formally, the discrete Langevin dynamic for sampling is typically formulated as
\begin{align}\label{eq: Langevin_discrete}
    \mathbf{x}_{t} = \mathbf{x}_{t-1} +  \frac{\epsilon_t^2}{2}\bigtriangledown_{\mathbf{x}}\log p_{\mathbf{x}^\ast}(\mathbf{x}_{t-1})  + \epsilon_t \mathbf{z}_t, 1\leq t \leq T
\end{align}
where $\epsilon_t$ is the step size (a positive real scalar), $\mathbf{z}_t$ is a standard Gaussian noise, and $T$ is the iterations. Starting from a standard Gaussian sample $\mathbf{x}_0$, with a total of $T$ steps this sequential sampling process gradually transforms $\mathbf{x}_0$ to the sample $\mathbf{x}_T$ that obeys the target distribution $p_{\mathbf{x}^{\ast}}$, which is the steady-state distribution of the process.
The discrete reverse diffusion for sampling is the reverse of a corrupting process starting from the data distribution $p_{\mathbf{x}^{\ast}}$
\begin{align}
    \mathbf{x}_t= \mathbf{x}_{t-1} +g_{t}\mathbf{z}_t,1\leq t \leq T,
\end{align}
where $g_{t}$ is the step size.
Typically, the discrete reverse diffusion is formulated as
\begin{align}\label{eq:discrete_reverse_diffusion}
\mathbf{x}_t= \mathbf{x}_{t-1} + \bar{g}_{t}^2 \bigtriangledown_{\mathbf{x}}\log \bar{p}_{t}(\mathbf{x}) +\bar{g}_{t}\mathbf{z}_t,1\leq t \leq T,
\end{align}
where $\bar{g}_{t}=g_{T-t}$, and $\bar{p}_{t}=p_{T-t}$ is the probabilistic distribution function of the state $\mathbf{x}_{T-t}$ in the corrupting process at the time point $t$.

Since the evaluation of the gradient terms
$\bigtriangledown_{\mathbf{x}}\log p_{\mathbf{x}^{\ast}}(\mathbf{x})$ (or $\bigtriangledown_{\mathbf{x}}\log \bar{p}_{t}(\mathbf{x})$) consumes the major time in the sampling process, for accelerating the sampling process, a straightforward way is to reduce the iterations $T$ and proportionally expand the step size $\epsilon_t$ simultaneously, so that the number of calculating the gradient terms decreases, whilst keeping the total update magnitude.
However, this often makes pretrained SGMs fail in image synthesis, as shown in Fig.~\ref{fig:ffhq}.

Conceptually, the Langevin dynamics defined in Eq.~\eqref{eq: Langevin_discrete} can be considered as a special case of Metropolis adjusted Langevin algorithm (MALA)~\citep{roberts2002langevin,welling2011bayesian,girolami2011riemann}.
When the coordinates of the target dataset (\eg, the pixel locations of a natural image) have widely varying scales or are strongly correlated, the sampling process becomes {\em ill-conditioned}. 
As a result, the convergence of the sampling process is not guaranteed when the step sizes are large~\citep{girolami2011riemann}. This explains the failure of sampling under direct acceleration (Fig.~\ref{fig:ffhq}).
{Noting} that the reverse diffusion has a similar update rule {to} the Langevin dynamics (Eq.~\eqref{eq: Langevin_discrete} vs. Eq.~\eqref{eq:discrete_reverse_diffusion}), we hypothesize that it suffers {from} the same ill-conditioned issue. 

\begin{figure*}[t]
    \centering
    \begin{minipage}{1\linewidth}
        \textcolor{white}{-----------------}$T = 2000$\textcolor{white}{---------------------}$T = 200$\textcolor{white}{---------------------}$T = 133$\textcolor{white}{---------------------}$T = 100$\textcolor{white}{---------------------}$T = 66$
    \end{minipage}

    \rotatebox{90}{\textcolor{white}{------------------} \quad\quad \quad  Ours \quad\quad\quad\quad\quad\quad\quad\quad \quad\quad\quad\quad\quad\quad  Baseline~\cite{song2020score}} 
    \includegraphics[width=0.975\textwidth]{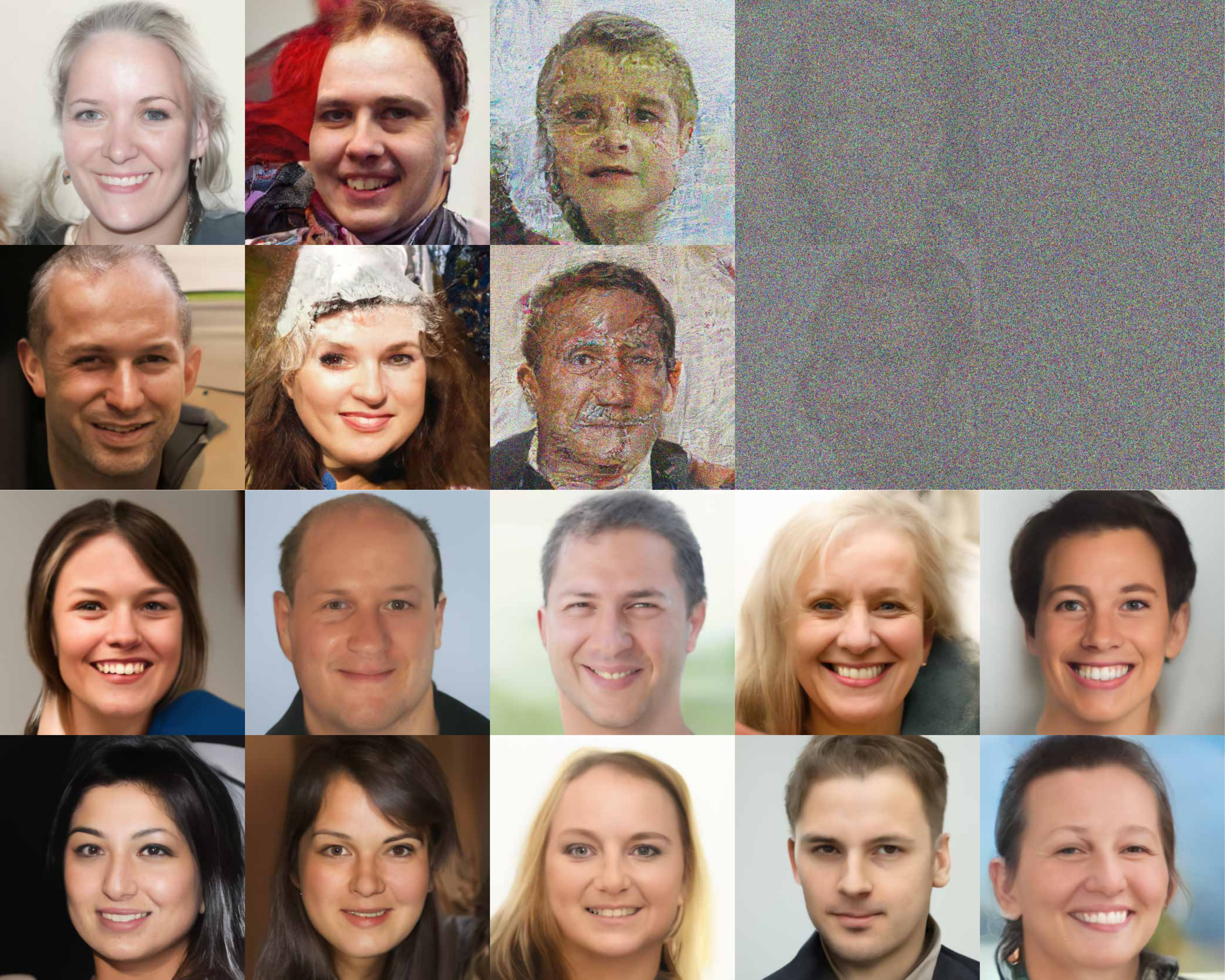}
    \caption{
        Facial images at a resolution of $1024\times 1024$ generated by NCSN++~\citep{song2020score}
        under a variety of sampling iterations
        (top) without and (bottom) with our PDS. 
        NCSN++ decades quickly with increasingly reduced sampling iterations, which can be well solved with PDS.
        For generating a batch of 8 images, PDS reduces the time cost of generating $8$ human facial samples from $1920$ sec (more than half an hours) to $68$ sec (about one minute) on one NVIDIA RTX 3090 GPU. Dataset: FFHQ \citep{karras2019style}. More samples in Appendix. }
        \label{fig:ffhq}
\end{figure*}

\begin{figure}[h]
    \centering
    \includegraphics[width=0.48\textwidth]{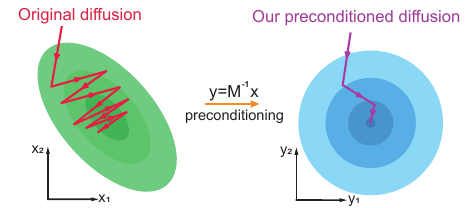}
    \caption{Illustration of our preconditioning method for accelerating sampling process.}
    \label{fig:demo}
\end{figure}

In light of this insight, we propose an efficient {\bf\em \fullmodelname}~(\shortmodelname) method for accelerating existing pretrained SGMs, without the need for model retraining.
The key idea is to make the scales of the target distribution more similar along all the coordinates~\citep{roberts2002langevin,li2016preconditioned} using a {\em matrix preconditioning} operation, hence solving the ill-conditioned problem of both the Langevin dynamics Eq.~\eqref{eq: Langevin_discrete} and the reverse diffusion Eq.~\eqref{eq:discrete_reverse_diffusion}, as demonstrated in Fig.~\ref{fig:demo}.
Formally, we enrich the Langevin dynamics Eq.~\eqref{eq: Langevin_discrete} by reformulating the sampling process as
\begin{align}\label{eq: Langevin_discrete_M}
    \mathbf{x}_{t} = \mathbf{x}_{t-1} +  \frac{\epsilon_t^2}{2}MM^{\mathrm{T}}\bigtriangledown_{\mathbf{x}}\log p_{\mathbf{x}^\ast}(\mathbf{x}_{t-1})  + \epsilon_t M\mathbf{z}_t,
\end{align}
where $M$ is the newly introduced preconditioning matrix designed particularly for regulating the behavior of accelerated sampling processes.
This proposed preconditioning matrix $M$ balances the scale of different coordinates of the sample space, hence alleviating the ill-conditioned issue.
For the reverse diffusion Eq.~\eqref{eq:discrete_reverse_diffusion}, we also equip with a similar preconditioning
\begin{align}
    \mathbf{x}_t= \mathbf{x}_{t-1} +\bar{g}_{t}^2 MM^{\mathrm{T}} \bigtriangledown_{\mathbf{x}}\log \bar{p}_{t}(\mathbf{x}_{t-1}) +\bar{g}_{t}M\mathbf{z}_t,
\end{align}
Theoretically, we prove that PDS can preserve both the steady-state distribution of the original Langevin dynamics and the final-state distribution of the original reverse diffusion, with the assistance of Fokker-Planck equation~\citep{gardiner1985handbook} and related techniques.
We construct the preconditioning operator $M$ using Fast Fourier Transform (FFT)~\citep{brigham1988fast}, making its computational cost marginal.

In this work, we make the following {\bf contributions}:
{\bf (1)} For sampling acceleration,
we introduce a novel \fullmodelname
~(\shortmodelname) process.
\shortmodelname~preconditions the existing sampling process 
for alleviating ill-conditioned issues, 
with {a} theoretical guarantee of 
keeping the original target distributions in convergence.
{\bf (2)}
With \shortmodelname, {various} pretrained SGMs can be accelerated significantly for image synthesis at various spatial resolutions, without model retraining.
In particular, \shortmodelname~delivers $28\times$ reduction in wall-clock time for high-resolution $(1024\times1024)$ image synthesis.
Compared with latest alternatives (\eg, CLD-SGM~\citep{dockhorn2021score}, DDIM~\citep{song2020denoising}, and Analytic-DDIM~\citep{bao2022analytic}), PDS achieves superior performance on CIFAR-10 (FID score 1.99). 

A preliminary version of this work was presented in ECCV 2022~\citep{ma2022accelerating}.
In this paper, we have further extended our conference version as follows:
{\bf (i)} We reveal that the sampling process of the original SGMs is insensitive to the structural data (Sec.~\ref{sec:lim});
{\bf (ii)} We provide more theoretical explanation including its effect of preserving the final-state distribution of the reverse diffusion during sampling (Sec.~\ref{sec:reverse_final});
{\bf (iii)} We develop indicators that further illustrate the mechanism of the PDS in terms of alleviating the ill-conditioning issue.
{\bf (iv)} We conduct more quantitative experiments, including an extra dataset (CelebA, COCO and ImageNet), more iteration budgets evaluation, and comparing with more acceleration methods in both SGMs and DDPMs;
{\bf (v)} We investigate the parameter space of PDS more comprehensively, resulting in further performance improvement;
{\bf (vi)} We develop a theory easing the process of estimating proper {parameters} (Sec.~\ref{sec:theory_relation}) with experimental validation (Sec.~\ref{sec:relation}).
{\bf (vii)} We implement sensitivity analysis on the parameter of PDS and reveal its robustness.

%% file: sections/2_related_work.tex
\section{Related work}
\cite{sohl2015deep} first proposed to sequentially corrupt a target data distribution with Gaussian noises and train a model to learn the backward process to recover the target data distribution, inspired by non-equilibrium statistical physics. 
Exploring this procedure later on, \cite{song2019generative} proposed {the first SGM}, the noise conditional score network (NCSN). 
They further proposed NCSNv2 
for higher resolution image generation (\eg, $256\times 256$) by scaling noises and improving stability with moving average~\cite{DBLP:conf/nips/0011E20}.
Summarizing all the previous SGMs into a unified framework based on the stochastic differential equation (SDE), \cite{song2020score} proposed NCSN++ to generate high-resolution images via numerical SDE solvers for the first time. \cite{de2021diffusion} provided the first quantitative convergence results for SGMs. 
As another class of relevant generative models, denoising diffusion probabilistic models (DDPMs)~\citep{DBLP:conf/nips/HoJA20,dhariwal2021diffusion,ho2021cascaded,nichol2021glide} also demonstrate excellent performance on image synthesis.
They are mostly trained by maximizing an evidence lower bound (ELBO).

There are some recent works on improving the sampling quality of SGMs and DDPMs by reshaping the diffusion process.
\cite{song2020denoising} proposed DDIMs by improving the sampling of DDPMs via a non-Markovian sampling process.
\cite{liu2021pseudo} further improved DDIMs by implementing the sampling process near a manifold to guarantee that the model can fit the ground-truth score function well.
Using the Hamiltonian Monte Carlo methods~\citep{neal2011mcmc}, \cite{dockhorn2021score} extended the sample space of SGMs with a velocity vector, and proposed critically-damped Langevin diffusion (CLD) based SGMs.
\cite{jolicoeur2021gotta} used a numerical SDE solver with adaptive step sizes to accelerate SGMs.

As our PDS, some prior methods utilize non-isotropic noises to accelerate the sampling process.
\cite{jing2022subspace} restricted the sampling process in a linear subspace where the data reside to accelerate the convergence of the sampling process. This might help with the {ill-condition} issue and convergence speed, as the space of sampling decreases. However, our PDS directly adjusts the scales of different coordinates of the sampling space to alleviate ill-condition issues, without the need for subspace restriction.
\cite{DBLP:conf/icml/NicholD21} learned an additional variance
vector during training a DDPM. Instead, we exploit a preconditioning matrix simply estimated from the training data, without training extra parameters. With diagonal covariance matrices, \cite{bao2022analytic,bao2022estimating} given the analytical formulation of a sequential of diagonal covariance matrices maximizing the ELBO between the corrupting process and the sampling process for better sample performance. In comparison, our method eliminates the need for solving additional complex optimization problems.

Some works focus on improving the sampling efficiency of SGMs or DDPMs by incorporating extra learnable modules.
Latent Diffusion Models (LDMs) use a nonlinear transformation induced by a variational autoencoder to transfer the sampling process of SGMs into a latent space~\citep{vahdat2021score, rombach2022high}. The lower-dimensional, decorrelated nature of this latent space helps mitigate ill-conditioned issues similar to PDS. However, unlike LDM, our PDS employs a train-free linear transform, offering a more cost-effective solution.
A number of recent works~\citep{salimans2022progressive,meng2023distillation,song2023consistency} utilized the concept of knowledge distillation to expedite the sampling process of DDPMs. Their common idea is a progressive distillation approach, involving the iterative refinement of a pretrained model.
\cite{kingma2021variational} proposed variations diffusion models to learn the noise schedule during training with another neural network for higher log-likelihood of the sampling results.
In comparison, our PDS directly leverages information from the data, avoiding the need for extra model training to accelerate sampling

Efforts to enhance the sampling quality of SGMs are also being made by modifying their training loss. \cite{song2021maximum} proposed reweighting the loss function with the diffusion coefficient of the SGMs. This enhancement aims to boost the log-likelihood of the model's generations, effectively improving the performance of VP (variance preserving) and subVP SGMs. However, this method cannot yield improvements for VE (variance exploding) SGMs.
Further development in this direction is~\citep{lu2022maximum}, where they proposed training the model to minimize {higher-order} score matching errors. This technique notably improves the performance of ordinary differential equation (ODE) type SGMs.
In comparison, our PDS demonstrates superior improvement on the stochastic differential equation (SDE) type VE (variance exploding) SGMs.

The above methods are limited in the following aspects:
{\bf (1)}
They tend to introduce much extra computation. For example, CLD-SGMs double the dimension of data {by} learning the velocity of diffusion. 
\cite{jolicoeur2021gotta} added a high-order numerical solver that increases the number of calling the SGM, resulting in much more cost. 
With our PDS, the only extra calculation relates the preconditioning matrix that can be efficiently implemented by Fast Fourier Transform.
As a result, PDS yields marginal extra computation to the vanilla SGMs.
{\bf (2)}
{They often improve a {\em single specific} SGM/DDPM while our PDS is able to accelerate three different SGMs (NCSN~\citep{song2019generative}, NCSNv2~\citep{DBLP:conf/nips/0011E20}, and NCSN++~\citep{song2020score}).}
{\bf (3)}
Unlike our work, none of {the} previous methods has demonstrated the scalability to more challenging high-resolution image generation tasks (\eg, FFHQ facial images with $1024\times1024$ resolution). 

%% file: sections/3_preliminaries.tex
\section{Preliminaries}\label{sec: pre}
\subsection{Score matching}
Assume that we observe i.i.d. samples from an unknown distribution $p_{\mathbf{x}^{\ast}}(\mathbf{x})$ of random vector $\mathbf{x}\in\mathbb{R}^d$, parameterized as $p_{\mathbf{x}^{\ast}}(\cdot;\theta)$ with $\theta$ the learnable parameters.
In a non-normalized statistical learning setting, we have an explicit expression for $p_{\mathbf{x}^{\ast}}(\cdot;\theta)$ up to a normalization factor: $p_{\mathbf{x}^{\ast}}(\mathbf{x};\theta) = \frac{1}{Z(\theta)}q(\mathbf{x};\theta),\forall \mathbf{x}\in\mathbb{R}^d$,
where $Z(\theta)=\int_{\mathbf{x}\in\mathbb{R}^d} q(\mathbf{x};\theta) d\mathbf{x}$ normalizes $q$ so that it becomes a well-defined probabilistic distribution function.
The normalization factor $Z(\theta)$ is often intractable both analytically and numerically, especially for high-dimensional ($d$) data.
This makes direct learning of $p_{\mathbf{x}^{\ast}}(\cdot;\theta)$ 
challenging.
Score matching is an effective approach to solving this challenge~\citep{hyvarinen2005estimation}. Specifically, a model $\mathbf{s}_{\theta}(\cdot)$ is trained to directly approximate the \textit{score function} $\mathbf{s}_{\mathbf{x}^\ast}(\cdot):=\bigtriangledown_{\mathbf{x}}\log p_{\mathbf{x}^{\ast}}(\cdot)$ through the objective
\begin{align}\label{eq:score_match_loss}
    \mathbb{E}_{\mathbf{x}^{\ast}\sim p_{\mathbf{x}^{\ast}}}[\left\|\mathbf{s}_{\theta}(\mathbf{x})-\mathbf{s}_{\mathbf{x}^\ast}(\mathbf{x})\right\|^2].
\end{align}
{In} the absence of $Z(\theta)$, training with this objective is more tractable.
\cite{hyvarinen2005estimation} further proves that under some mild conditions, if $p_{\mathbf{x}^{\ast}}(\cdot)=p_{\mathbf{x}^{\ast}}(\cdot;\theta^{\ast})$, 
the ground-truth parameters $\theta^\ast$
can be achieved by minimizing Eq.~\eqref{eq:score_match_loss} to zero.

However, Eq.~\eqref{eq:score_match_loss} can not be evaluated in more general cases where an expression of the score function about $\theta$ is absent. To address this, \cite{vincent2011connection} proposes an equivalent, simpler objective inspired by denoising autoencoders
\begin{align}\label{eq:score_denoise_loss}
     \mathbb{E}_{\mathbf{x}\sim p_{\mathbf{x}^{\ast}}}[\left\|\mathbf{s}_{\theta}(\mathbf{x}_{\sigma})-\bigtriangledown_{\mathbf{x}_{\sigma}}\log p_{\sigma}(\mathbf{x}_{\sigma}|\mathbf{x}) \right\|^2],
\end{align}
where $p_{\sigma}(\mathbf{x}_{\sigma}|\mathbf{x})=\mathcal{N}(\mathbf{x}_{\sigma};\mathbf{x},\sigma^2I)$. Formally, given a noised data sample $\mathbf{x}_{\sigma}$, the model $\mathbf{s}_{\theta}(\cdot)$ is trained to extract the Gaussian noise vector added into the sample. This denoises the sample corrupted by the noise, \ie, subtracting the noise vector gives the original sample.
It is proved that minimizing the objective Eq.~\eqref{eq:score_denoise_loss} allows the model to fit the ground-truth score function $\mathbf{s}_{\mathbf{x}^\ast}$ of the data distribution~\citep{vincent2011connection}.

\subsection{Langevin dynamics}
SGMs generate samples from a target distribution by iteratively evaluating the target score function.
The first SGM, noise conditional score network (NCSN) \citep{song2019generative} $\mathbf{s}_{\theta}$, is trained on a series of score matching tasks via
\begin{align*}
    \sum_{t=1}^T\sigma^2_t \mathbb{E}_{\mathbf{x}\sim p_{\mathbf{x}}}\left[\left\|\mathbf{s}_{\theta}(\mathbf{x}_{\sigma_t},\sigma_t)-\bigtriangledown_{\mathbf{x}_{\sigma_t}}\log p_{\sigma_t}(\mathbf{x}_{\sigma_t}|\mathbf{x}) \right\|^2\right],
\end{align*}
with the standard deviations $\{\sigma_t\}_{t=1}^T$ satisfies $0<\sigma_1 < \sigma_2 < \ldots < \sigma_T$.
{During} inference, an NCSN generates samples from $p_{\mathbf{x}^{\ast}}$ through the discretized Langevin dynamics 
\begin{align}\label{eq: reverse}
    \mathbf{x}_{t} = \mathbf{x}_{t-1} +  \frac{\epsilon_t^2}{2}\bigtriangledown_{\mathbf{x}_{t-1}}\log p_{\mathbf{x}^{\ast}}(\mathbf{x}_{t-1}) + \epsilon_t \mathbf{z}_t, 1\leq t \leq T,
\end{align}
where the initialization $\mathbf{x}_0\sim \mathcal{N}(\mathbf{0},I)$, $\log p_{\mathbf{x}^{\ast}}(\mathbf{x}_{t-1})$ is approximated by $\mathbf{s}_{\theta}(\mathbf{x}_{t-1},\sigma_i)$ at a proper level of $\sigma_i$, $\epsilon_t>0$  the step size, and $T$ the iterations.
With this process, we transform a Gaussian sample $\mathbf{x}_0$ to $\mathbf{x}_T$ in the target distribution $p_{\mathbf{x}^{\ast}}$.
{Note that we define our notations for the forward and reverse processes in the manner that the process begins with \(\mathbf{x}_0\) and ends with \(\mathbf{x}_T\).}
The Langevin dynamics Eq.~\eqref{eq: reverse} works in that 
as $T\rightarrow+\infty$, the distribution of $\mathbf{x}_t$ converges to a {\em steady-state distribution}, exactly $p_{\mathbf{x}^{\ast}}$~\citep{girolami2011riemann}..
Improvements can be made
by refining noise scales, iterations, and step {size},
and using the moving average technique~\citep{DBLP:conf/nips/0011E20}.

\subsection{Reverse diffusion}
Given the discretized corrupting (\ie, forward) process
\begin{align}\label{eq:forward_diffusion}
    \mathbf{x}_0\sim p_{\mathbf{x}^{\ast}},\mathbf{x}_t=\mathbf{x}_{t-1} + g_t\mathbf{z}_t, 1\leq t \leq T,
\end{align}
where $g_t = \sqrt{\sigma_t^2-\sigma_{t-1}^2}$,
an alternative generation strategy is by simulating a corresponding discretized reverse process
\begin{align}\label{eq:reverse_diffusion_discrete}
   \mathbf{x}_t= \mathbf{x}_{t-1} + \bar{g}_{t}^2 \bigtriangledown_{\mathbf{x}}\log \bar{p}_{t}(\mathbf{x}) +\bar{g}_{t}\mathbf{z}_t,1\leq t \leq T,
\end{align}
where the initial distribution $\mathbf{x}_0\sim \mathcal{N}(\mathbf{0},I)$, $\bar{g}_{t}=g_{T-t}$, $\bar{p}_{t}={p}_{T-t}$ is the probabilistic distribution function of $\mathbf{x}_{T-t}$ in the forward process.
In this way, the final-steady state $\mathbf{x}_T$ of reverse diffusion obeys $p_{\mathbf{x}^\ast}$~\citep{haussmann1986time}.
Eq.~\eqref{eq:forward_diffusion} and \eqref{eq:reverse_diffusion_discrete} both are the discretized version of the corresponding stochastic differential equation (SDE).
Hence, a numerical solver of SDE can be used to implement the sampling process, as realized in NCSN++ \citep{song2020score}. 
Specifically, NCSN++ learns to fit the gradient term $\bigtriangledown_{\mathbf{x}}\log p_{T-t}(\mathbf{x})$ through minimizing 
\begin{align*}
       \mathbb{E}_{t}\lambda_t\mathbb{E}_{\mathbf{x}_0\sim p_{\mathbf{x}^\ast}}\left[\left\|\mathbf{s}_{\theta}(\mathbf{x}(t),t)-\bigtriangledown_{\mathbf{x}_{t}}\log p_{t}(\mathbf{x}_t|\mathbf{x}_0) \right\|^2\right], 
\end{align*}
where $t$ is uniformly drawn from $\{1,\ldots,T\}$, and $\lambda_t$ is the weight at $t$.
During sampling, NCSN++ updates the state in two steps (Alg.~\ref{alg:original_sample}): (1) predicting with the reverse diffusion (Eq.~\eqref{eq:reverse_diffusion_discrete}),
(2) correcting with the Langevin dynamics (Eq.~\eqref{eq: reverse}). The variance exploding (VE) variant is the best among all NCSN++ models~\citep{song2020score}.

\begin{algorithm}[h]
  \caption{Sampling process of VE NCSN++}
  \label{alg:original_sample}
\begin{algorithmic}
  \State {\bfseries Input:} 
  The sampling iterations $T$, the score model $\mathbf{s}_{\theta}$;
  \State {\bfseries Initialization:} Draw $\mathbf{x}_0 \sim\mathcal{N}(\mathbf{0},I)$;
  \For{$t=1$ {\bfseries to} $T$} 
        \State \textbf{Predicting (reverse diffusion):} 
        \State Drawing $\mathbf{z}_t\sim\mathcal{N}(\mathbf{0},I)$
     \State $\mathbf{x}'_{t} \leftarrow \mathbf{x}_{t-1} +\bar{g}_{t}^2\mathbf{s}_{\theta}(\mathbf{x}_{t-1},\bar{g}_{t})+\bar{g}_{t}\mathbf{z}_t$
        \State \textbf{Correcting (Langevin dynamics):} 
        \State Drawing $\mathbf{z}_t\sim\mathcal{N}(\mathbf{0},I)$
        \State $\mathbf{x}_{t} \leftarrow \mathbf{x}'_{t} + \frac{\epsilon_t^2}{2}\mathbf{s}_{\theta}(\mathbf{x}'_{t},\epsilon_t)+\epsilon_t\mathbf{z}_t$
  \EndFor
  \State {\bfseries Output:} $\mathbf{x}_T$
\end{algorithmic}
\end{algorithm}

\subsection{Limitation analysis}\label{sec:lim}
While having superior generating capability over GANs, SGMs 
run much slower.
A direct way for acceleration is to reduce the iterations $T$ and increase the step size $\epsilon_t$ proportionally (\ie, one step for multiple ones).
This however leads to severe degradation (see Fig.~\ref{fig:ffhq}).
This is partly because the sampling process is fundamentally limited in sampling structural data (\eg, images), as elaborated below.
\begin{theorem}\label{Thm: transform}
Suppose $U\in \mathbb{R}^{d\times d}$ is an orthogonal matrix, then the Langevin dynamics Eq.~\eqref{eq: Langevin_discrete}
can be rewritten as
\begin{align}\label{eq:freq_diffusion}
\tilde{\mathbf{x}}_{t}  = \tilde{\mathbf{x}}_{t-1} + \frac{\epsilon^2_t}{2}\bigtriangledown_{\tilde{\mathbf{x}}}\log p_{\tilde{\mathbf{x}}^{\ast}}(\tilde{\mathbf{x}}_{t-1})+\epsilon_tU\mathbf{z}_t,
\end{align}
where $\tilde{\mathbf{x}} = U\mathbf{x}$. Similarly, the reverse diffusion Eq.\eqref{eq:reverse_diffusion_discrete} can also be rewritten as
\begin{align}\label{eq:freq_reverse_diffusion}
    &\tilde{\mathbf{x}}_{t} = \tilde{\mathbf{x}}_{t-1} + \bar{g}_{t}^2 \bigtriangledown_{\tilde{\mathbf{x}}_{t} }\log \bar{p}_{t}(\tilde{\mathbf{x}}_{t} ) +\bar{g}_{t}U\mathbf{z}_t
\end{align}
\end{theorem}
\begin{proof}
Multiplying $U$ at both sides of the original Langevin dynamics, we have
\begin{align}\label{eq:times_U}
    \tilde{\mathbf{x}}_{t-1}  = \tilde{\mathbf{x}}_t + \frac{\epsilon^2_t}{2}U\bigtriangledown_{\mathbf{x}}\log p_{\mathbf{x}^{\ast}}(\mathbf{x}_t)+\epsilon_tU\mathbf{z}_t.
\end{align}
Using the change-of-variable technique, we have
\begin{align*}
 p_{\tilde{\mathbf{x}}^{\ast}}(\tilde{\mathbf{x}}) = p_{\mathbf{x}^{\ast}}(U^{-1}\tilde{\mathbf{x}})\left | U \right |^{-1}.
\end{align*}
By taking the logarithm and calculating the gradients at both sides of the equation, we have
\begin{align}\label{eq:trans}
    \bigtriangledown_{\tilde{\mathbf{x}}}\log p_{\tilde{\mathbf{x}}^{\ast}}(\tilde{\mathbf{x}}) = \bigtriangledown_{\tilde{\mathbf{x}}}\log p_{\mathbf{x}^{\ast}}(\mathbf{x})
\end{align}
Combining Eq.~\eqref{eq:trans}, the chain rule $\bigtriangledown_{\mathbf{x}} = U^{\mathsf{T}}\bigtriangledown_{\tilde{\mathbf{x}}}$, and Eq.~\eqref{eq:times_U}, we obtain
\begin{align*}
    \tilde{\mathbf{x}}_{t-1}  = \tilde{\mathbf{x}}_t + \frac{\epsilon_t^2}{2}UU^{\mathsf{T}}\bigtriangledown_{\tilde{\mathbf{x}}}\log p_{\tilde{\mathbf{x}}^{\ast}}(\tilde{\mathbf{x}}_t)+\epsilon_tU\mathbf{z}_t.
\end{align*}
As $U$ is orthogonal, $UU^{\mathsf{T}} = I_d$. We thus finish the proof for the Langevin dynamics Eq.~\eqref{eq: reverse}. The proof of the reverse diffusion Eq.~\eqref{eq:reverse_diffusion_discrete} is almost the same.
\end{proof}
According to this theorem, we can directly transform the entire sampling process by an orthogonal transformation $U$, with the only difference that the noise is altered to $U\mathbf{z}_t$. Particularly, if $\mathbf{z}_t$ obeys an isotropic Gaussian distribution applied by existing methods,
$U\mathbf{z}_t$ obeys exactly the same distribution as $\mathbf{z}_t$. 
Hence, the sampling process is {\bf \em orthogonally invariant}.

Since any permutation matrix (which operates on a vector by changing the order of its coordinates) is orthogonal, Thm.~\ref{Thm: transform} implies that the sampling process remains equivalent regardless of how the coordinates of \(\mathbf{x}_t\) are reordered. However, for structured data (e.g., images), the order of coordinates encodes both local geometric information of each pixels and global semantic structures. Therefore, the current sampling process in Eq.~\eqref{eq: reverse} disregards the structural integrity of the data.

\subsection{Preconditioning}
Conceptually, the Langevin dynamics Eq.~\eqref{eq: reverse} of the sampling process is a special case of Metropolis adjusted Langevin algorithm (MALA)~\citep{roberts2002langevin,welling2011bayesian,girolami2011riemann} with non-autonomous factor $\epsilon_t$.
Inspired by the findings that the preconditioning approach can significantly accelerate the MALA~\citep{roberts2002langevin}, we study {\em how to construct a preconditioning operator $M$ to accelerate the sampling phase of SGMs}.

Specifically, the original sample space may suffer from an ill-conditioned issue caused by the large variations of the scales in different coordinates. This could lead to oscillation and slow convergence, especially when the step size $\epsilon_t$ is not small enough~\citep{nocedal1999numerical}.
Preconditioning can make the scales become more similar along all the coordinates~\citep{roberts2002langevin,li2016preconditioned} using a {\em matrix preconditioner}, alleviating the ill-conditioned problem and improving the convergence of the sampling process even under a large step size.
This applies to the reverse diffusion (Eq.~\eqref{eq:reverse_diffusion_discrete}) similarly.
As shown in the following sections, with the proposed preconditioning, the noise terms $\mathbf{z}_t$ become anisotropic (\ie, its covariance matrix is no longer a scalar matrix). As a result, the sampling process of SGMs {varies} under the permutation of coordinates, becoming conditional to the coordinate order (Thm.~\ref{Thm: transform}). 

%% file: sections/4_theory.tex
\section{Theoretical analysis}
In this section, we develop the theories for the formulation of our preconditioning operator suitable for the sampling process {incorporating prediction and correction} (Alg.~\ref{alg:original_sample}).
{Since the sampling process consists of two types of stochastic processes—prediction (reverse diffusion) and correction (Langevin dynamics)—we must demonstrate that the preconditioning operator preserves the steady-state distribution of each process separately.}

\subsection{Continuous processes}
For theoretical purposes, we first give the continuous version of sampling
discussed in Sec.~\ref{sec: pre}.
The continuous version of Eq.~\eqref{eq: reverse} is
\begin{align}\label{eq:reverse_con}
        d\mathbf{x}_t =  \frac{\epsilon_t^2}{2}\bigtriangledown_{\mathbf{x}_t}\log p_{\mathbf{x}^{\ast}}(\mathbf{x}_t)dt  + \epsilon_t d\mathbf{w}_t,t\in[0,T]
\end{align}
starting from an initial distribution $\mathbf{x}_0\sim\mathcal{N}(\mathbf{0},I)$, with {$\mathbf{w}_t$} a Wiener process, and $\epsilon_t$ {is} the continuous version of the step size.
The continuous version for corrupting process (Eq.~\eqref{eq:forward_diffusion}) is
\begin{align}\label{eq:forward_diffusion_con}
    d\mathbf{x}_t=g_td\mathbf{w}_t,t\in[0,T],
\end{align}
 starting from an initial distribution $\mathbf{x}_0\sim p_{\mathbf{x}^{\ast}}$.
The continuous version of reverse diffusion {is}(Eq.~\eqref{eq:reverse_diffusion_discrete})
\begin{align}\label{eq:reverse_diffusion}
    d\mathbf{x}_t= \bar{g}_{t}^2 \bigtriangledown_{\mathbf{x}}\log \bar{p}_{t}(\mathbf{x}_t) +\bar{g}_{t}d\mathbf{w}_t,t\in[0,T]
\end{align}
starting from an initial distribution $\mathbf{x}_0\sim p_{\mathbf{x}^{\ast}}$. Note that these processes result in a different distribution for the initial state $\mathbf{x}_0$.

\subsection{Preserving the steady-state distribution {of} the Langevin dynamics}

To ensure generation quality, the steady-state distribution of the Langevin dynamics (Eq.~\eqref{eq:reverse_con}) {must match} the target distribution \(p_{\mathbf{x}^{\ast}}\). Failure to satisfy this condition may lead to poor performance. We prove that the proposed preconditioning preserves the steady-state distribution.

We first derive the theory behind the Langevin dynamics on why it can generate samples from $p_{\mathbf{x}^\ast}$.
According to~\citep{gardiner1985handbook}, Eq.~\eqref{eq:reverse_con} is associated with a Fokker-Planck equation
\begin{align}\label{eq: Langevin_fp}
    \frac{\partial p}{\partial t} = - \frac{\epsilon_t^2}{2}\bigtriangledown_{\mathbf{x}}\cdot (\bigtriangledown_{\mathbf{x}}\log p_{\mathbf{x}^{\ast}}(\mathbf{x}_t)p)+\frac{\epsilon_t^2}{2}\Delta_{\mathbf{x}} p,
\end{align}
where $p=p(\mathbf{x},t)$ is the distribution of $\mathbf{x}$ at time $t$.
This describes the stochastic dynamics of Eq.~\eqref{eq:reverse_con}.
Let $\frac{\partial p}{\partial t} = 0$, we obtain the steady-state function
\begin{align}\label{eq: Langevin_steady}
   \bigtriangledown_{\mathbf{x}}\cdot (\bigtriangledown_{\mathbf{x}}\log p_{\mathbf{x}^{\ast}}(\mathbf{x})p) = \Delta_{\mathbf{x}} p.
\end{align}
Its solution, \ie, the steady-state solution, corresponds to the probabilistic density function of the steady-state distribution of Eq.~\eqref{eq:reverse_con}. Substituting $p$ with $p_{\mathbf{x}^{\ast}}$ in Eq.~\eqref{eq: Langevin_steady}, the L.H.S. of this equation becomes
\begin{align*}
    \bigtriangledown_{\mathbf{x}}\cdot (\bigtriangledown_{\mathbf{x}}\log (p_{\mathbf{x}^{\ast}})p_{\mathbf{x}^{\ast}}) = \bigtriangledown_{\mathbf{x}}\cdot (\bigtriangledown_{\mathbf{x}}(p_{\mathbf{x}^{\ast}})) =  \Delta_{\mathbf{x}} p^\ast,
\end{align*}
which is exactly the R.H.S. of the equation. Hence, the target data distribution $p_{\mathbf{x}^\ast}$ is one of the steady-state solutions of Eq.~\eqref{eq:reverse_con}.
It is proved that this  steady-state solution is {\em unique}, {\em s.t.} $\epsilon_t>0$~\citep{risken1985fokker}. As a result, with enough iterations, the state distribution will converge to the target data distribution $p_{\mathbf{x}^{\ast}}$.

The Fokker-Planck equation indicates how to preserve the steady-state distribution of the original process under modification. Concretely, we can theoretically show that the steady-state distribution remains when changing Eq.~\eqref{eq:reverse_con} as
\begin{align}\label{eq: Langevin_M}
    d\mathbf{x}_t = \frac{\epsilon_t^2}{2}(MM^{\mathrm{T}}+S)\bigtriangledown_{\mathbf{x}_t}\log p_{\mathbf{x}^{\ast}}(\mathbf{x}_t)dt + \epsilon_t Md\mathbf{w}_t,
\end{align}
where $M$ is an invertible linear operator, $M^\mathrm{T}$ is its adjoint operator, and $S$ is a skew-symmetric operator.
\begin{theorem}\label{thm: unchange}
The steady-state distribution of Eq.~\eqref{eq: Langevin_M} is $p_{\mathbf{x}^{\ast}}$, as long as the linear operator $M$ is invertible and the linear operator $S$ is skew-symmetric.
\end{theorem}
\begin{proof}
The invertibility of $M$ preserves the {uniqueness} of the solution of the steady-state equation~\citep{gardiner1985handbook}.
Hence, we only need to prove that the target data distribution $p_{\mathbf{x}^{\ast}}$ still satisfies the steady-state equation after preconditioning.
The Fokker-Planck equation of Eq.~\eqref{eq: Langevin_M} is
\begin{align*}
        &\frac{\partial p}{\partial t} = - \frac{\epsilon^2}{2}\bigtriangledown_{\mathbf{x}_t}\cdot (MM^{\mathrm{T} }\bigtriangledown_{\mathbf{x}_t}\log p_{\mathbf{x}^{\ast}}(\mathbf{x})p)  \nonumber\\ &-\frac{\epsilon^2}{2}\bigtriangledown_{\mathbf{x}_t}\cdot  (S\bigtriangledown_{\mathbf{x}_t}\log p_{\mathbf{x}^{\ast}}(\mathbf{x})p)+\frac{\epsilon^2}{2}\bigtriangledown_{\mathbf{x}_t}\cdot (MM^{\mathrm{T} }\bigtriangledown_{\mathbf{x}_t} p).
\end{align*}
The corresponding steady-state equation is
\begin{align*}
&\bigtriangledown_{\mathbf{x}}\cdot (MM^{\mathrm{T} }\bigtriangledown_{\mathbf{x}}\log p_{\mathbf{x}^{\ast}}(\mathbf{x})p)+\bigtriangledown_{\mathbf{x}}\cdot  (S\bigtriangledown_{\mathbf{x}}\log p_{\mathbf{x}^{\ast}}(\mathbf{x})p)\\ &=\bigtriangledown_{\mathbf{x}}\cdot (MM^{\mathrm{T} }\bigtriangledown_{\mathbf{x}} p).  
\end{align*}
Set $p=p_{\mathbf{x}^{\ast}}$, the above equation becomes
\begin{align*}
    &\bigtriangledown_{\mathbf{x}}\cdot (MM^{\mathrm{T} }\bigtriangledown_{\mathbf{x}}\log( p_{\mathbf{x}^{\ast}})p_{\mathbf{x}^{\ast}})+\bigtriangledown_{\mathbf{x}}\cdot  (S\bigtriangledown_{\mathbf{x}}\log( p_{\mathbf{x}^{\ast}})p_{\mathbf{x}^{\ast}})\\ &=\bigtriangledown_{\mathbf{x}}\cdot (MM^{\mathrm{T} }\bigtriangledown_{\mathbf{x}} p_{\mathbf{x}^{\ast}}). 
\end{align*}
The first term in L.H.S. equals to the \textit{R.H.S.}, since
\begin{align*}
     \bigtriangledown_{\mathbf{x}}\cdot (MM^{\mathrm{T} }\bigtriangledown_{\mathbf{x}}p_{\mathbf{x}^{\ast}} \frac{1}{p_{\mathbf{x}^{\ast}}}p_{\mathbf{x}^{\ast}})
     = \bigtriangledown_{\mathbf{x}}\cdot (MM^{\mathrm{T} }\bigtriangledown_{\mathbf{x}}p_{\mathbf{x}^{\ast}}).
\end{align*}
Additionally, the second term in the \textit{L.H.S}. equals to 
\begin{align*}
    \bigtriangledown_{\mathbf{x}}\cdot  (S\bigtriangledown_{\mathbf{x}}p_{\mathbf{x}^{\ast}}) =   (\bigtriangledown_{\mathbf{x}}\cdot  S\bigtriangledown_{\mathbf{x}})p_{\mathbf{x}^{\ast}},
\end{align*}
where the differential operator
\begin{align*}
    \bigtriangledown_{\mathbf{x}}\cdot  S\bigtriangledown_{\mathbf{x}} = \sum_{ij}\frac{\partial}{\partial x_i}\frac{\partial}{\partial x_j}S_{ij} = 0
\end{align*}
since $S$ is skew-symmetric. Therefore, the term $\bigtriangledown_{\mathbf{x}}\cdot  (S\bigtriangledown_{\mathbf{x}}p_{\mathbf{x}^{\ast}})=0$.
Then, the steady-state solution of Eq.~\eqref{eq: Langevin_steady}, \ie, $p_{\mathbf{x}^{\ast}}$, also satisfies the steady-state equation of Eq.~\eqref{eq: Langevin_M}. As a result, the theorem is proved.
\end{proof}
Thm.~\ref{thm: unchange} provides a wide manipulative space for us to design the preconditioning operator $M$ and $S$ while preserving the steady-state distribution of $\mathbf{x}_T$ in the Langevin dynamics simultaneously. {More precisely, given a trained SGM with the sampling process described by Eq.~\eqref{eq:reverse_con}, the sampling process can be modified to Eq.~\eqref{thm: unchange} without altering the final generated distribution, provided that the step size is sufficiently small.
}
 
\subsection{Preserving the final-state distribution {of} the reverse diffusion}\label{sec:reverse_final}
{Although seemingly similar, the reverse diffusion (Eq.~\eqref{eq:reverse_diffusion}) is different from the Langevin dynamics (Eq.~\eqref{eq:reverse_con}) as $\bar{p}_t(\mathbf{x})\neq p_{\mathbf{x}^{\ast}}(\mathbf{x})$ in general. Therefore, we need to reanalyze how to preserve the steady-state distribution of the reverse diffusion process when introducing alteration to the process.
As the reverse diffusion (Eq.~\eqref{eq:reverse_diffusion}) directly reverses the corrupting process, its final-steady state $\mathbf{x}_T$ obeys $p_{\mathbf{x}^\ast}$ naturally~\citep{haussmann1986time}.}
Note {that} the final-steady distribution of Eq.~\eqref{eq:reverse_diffusion} may not be its steady-state distribution. 
This can be seen if we write out its steady-state equation through {the} Fokker-Planck equation
\begin{align*}
       2\bigtriangledown_{\mathbf{x}}\cdot (\bigtriangledown_{\mathbf{x}}\log p_{\mathbf{x}^\ast}(\mathbf{x})p) = \Delta_{\mathbf{x}} p.
\end{align*}
It is obvious that the target data distribution $p_{\mathbf{x}^{\ast}}$ does not satisfy this equation.

Accordingly, to enable the reverse diffusion {to} generate samples from $p_{\mathbf{x}^\ast}$ after preconditioning, we must preserve its final-state distribution $\mathbf{x}_T$.
To that end, we modify Eq.~\eqref{eq:reverse_diffusion} as
\begin{align}\label{eq:reverse_diffusion_M}
d\mathbf{x}_t= \bar{g}_{t}^2MM^T \bigtriangledown_{\mathbf{x}_t}\log \bar{p}_{t}(\mathbf{x}_t)
    +\bar{g}_{t}Md\mathbf{w}_t,t\in[0,T],
\end{align}
In fact, we have
\begin{theorem}\label{thm:unchange2}
The final-state distribution of Eq.~\eqref{eq:reverse_diffusion_M}
is also $p_{\mathbf{x}^\ast}$, subject to $M$ is an invertible linear operator.
\end{theorem}
\begin{proof}
Denote $p_{\mathbf{y}^\ast}$ as the probabilistic distribution function of $\mathbf{y} =M^{-1}\mathbf{x}$, where $\mathbf{x}\sim p_{\mathbf{x}^\ast}$.
We consider the corrupting process for $p_{\mathbf{y}^{\ast}}(\mathbf{y})$.
Define the transformed corrupting process (starting from $\mathbf{y}_0\sim p_{\mathbf{y}^\ast}$)
\begin{align}\label{eq:forward_M}
    d\mathbf{y}_t = M^{-1} d\mathbf{w}_t,t\in [0,T].
\end{align}
Denote $\tilde{\bar{p}}_{t}(\mathbf{y})$ as the probabilistic distribution function of the $\mathbf{y}_{T-t}$ in the process Eq.~\eqref{eq:forward_M}.
Denote $\mathbf{y}_t=M^{-1}\mathbf{x}_t$, and {substitute} $\mathbf{y}_t$ into Eq.~\eqref{eq:reverse_diffusion_M}, we have
\begin{align}\label{eq:reverse_diffusion_M2}
       d\mathbf{y}_t= \bar{g}_{t}^2M^T \bigtriangledown_{\mathbf{x}_t}\log \bar{p}_{t}(\mathbf{x}_t) +\bar{g}_{t}d\mathbf{w}_t,t\in[0,T],
\end{align}
where we have operated $M^{-1}$ at both sides of the equation and use the relation $\mathbf{y}=M^{-1}\mathbf{x}$.
We show that Eq.~\eqref{eq:reverse_diffusion_M2} is the reverse process of Eq.~\eqref{eq:forward_M}.
With the change-of-variable relationship of the probabilistic distribution function
\begin{align}\label{eq: chain rule}
    \bar{p}_{t}(\mathbf{x}_t) = \tilde{\bar{p}}_{t}(\mathbf{y}_t)\left |\det\left[ M^{-1}\right] \right |,
\end{align}
and the chain rule $\bigtriangledown_{\mathbf{x}_t} = M^{-T}\bigtriangledown_{\mathbf{y}_t}$
we have
\begin{align*}
    M^{-T}\bigtriangledown_{\mathbf{x}_t}\log \bar{p}_{t}(\mathbf{x}_t) = \bigtriangledown_{\mathbf{y}_t}\log \tilde{\bar{p}}_{t}(\mathbf{y}_t).
\end{align*}
In Eq.~\eqref{eq: chain rule}, $M^{-1}$ is considered as a matrix, and $\det\left[ M^{-1}\right]$ is its discriminant. Therefore, Eq.~\eqref{eq:reverse_diffusion_M2} can be written as
\begin{align*}
      d\mathbf{y}_t= \bar{g}_{t}^2\bigtriangledown_{\mathbf{y}_t}\log \tilde{\bar{p}}_{t}(\mathbf{y}_t) +\bar{g}_{t}d\mathbf{w}_t,t\in[0,T],
\end{align*}
which is exactly the reverse process of Eq.~\eqref{eq:forward_M}~\citep{haussmann1986time}. As a result, the final-state distribution of Eq.~\eqref{eq:reverse_diffusion_M2} obeys $\mathbf{y}_T\sim p_{\mathbf{y}^\ast}$.
Reversing $M$ on Eq.~\eqref{eq:reverse_diffusion_M2}, the final-state distribution of Eq.~\eqref{eq:reverse_diffusion_M} thus obeys $\mathbf{x}_T\sim p_{\mathbf{x}^\ast}$.
\end{proof}
Similar {to} Thm.~\ref{thm: unchange}, Thm.~\ref{thm:unchange2} provides a wide space for designing the preconditioner $M$ while preserving the final-state distribution of $\mathbf{x}_T$ in the reverse diffusion.

%% file: sections/5_pds.tex
\section{Preconditioned diffusion sampling}
Thus far, we have presented the generic preconditioning formulation for both Langevin dynamics and reverse diffusion. Before designing the preconditioned diffusion sampling for image generation, we first test its efficiency through a toy example. 

\subsection{{Toy example}}
We implement Langevin dynamics sampling to generate samples from a two-dimensional Gaussian distribution $\mathcal{N}(\mathbf{0},C)$, where the covariance matrix is $C = \begin{bmatrix} 9.0 & 0.2 \\ 0.2 & 0.5 \end{bmatrix}$. The condition number of the covariance matrix is $18.18$, indicating ill-conditioning. As {a} common preconditioning strategy, we set the preconditioning matrix for this sampling process as $C+I$. As shown in the left and middle of Fig.~\ref{fig:toy}, if we increase the step size, the original Langevin dynamics fail to generate samples of high probability. Whereas the preconditioned Langevin dynamics produces samples with much higher probability with large step size, as shown in the right of Fig.~\ref{fig:toy}.

\begin{figure}[h]
\centering
\includegraphics[width=0.48\textwidth]{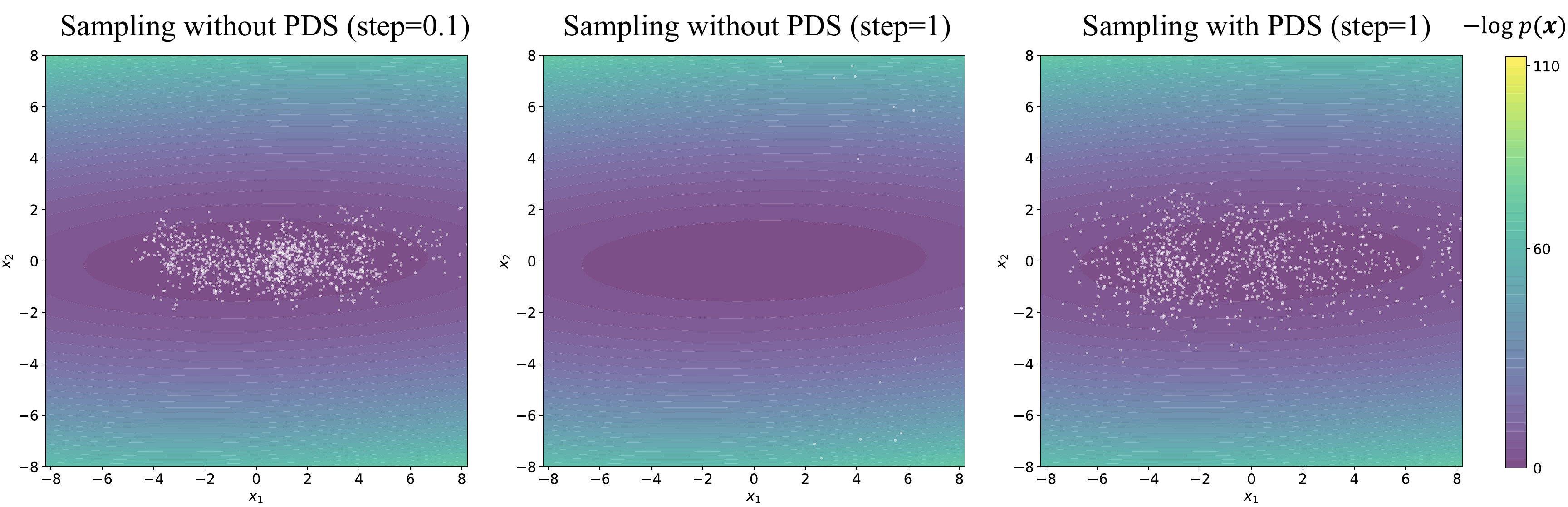}
        \caption{{Sampling on a two-dimensional Gaussian distribution using Langevin dynamics with step size $0.1$ (left), Langevin dynamics with step $1$ (middle) and preconditioned Langevin dynamics with step $1$ (right).}}
        \label{fig:toy}
\end{figure}

{Next, we focus on image generation tasks by designing proper operators. Unlike the toy model, image data are generally high-dimensional. Using the same strategy as the toy model—relying on the data's covariance matrix—would significantly increase computational costs during sampling, as matrix multiplication becomes computationally expensive in high-dimensional spaces. To overcome this, we leverage the structure of image data and design specialized preconditioners that enable efficient computation.} 

\subsection{Image preconditioning strategies}
In numerical optimization, accurate preconditioning operators
is calculated from the inverse of the Hessian matrix of $\log p_{\mathbf{x}^\ast}(\mathbf{x})$ at each time step. This however is extremely costly for high dimensional tasks (\eg, image generation)~\citep{wright1999numerical}. To bypass this obstacle, we adopt a data-centred strategy that leverages the coordinate scale information of target data.

For imagery data, we consider two types of coordinates: individual pixels (\ie, the spacial coordinate) and frequency (\ie, Fourier transformed). The former is useful for images with global structure characteristics, as per-pixel statistics gives prior information (\ie, {\bf \em pixel preconditioner}). 
With natural images, the amplitude of low-frequency components is typically much higher than that of high-frequency ones~\citep{Schaaf1996ModellingTP}.
This structural frequency property could be similarly useful. 
(\ie, {\bf \em frequency preconditioner}). 
For either case, we construct the linear operator $M$ with the mean value of target training data, used to balance all the coordinates (Fig.~\ref{fig:demo}). We set the skew-symmetric operator $S$ (Eq.~\eqref{eq: Langevin_M}) as zero and will discuss it later.
Combining the two preconditioners, we form a linear preconditioning operator $M\in \mathbb{R}^{d}\rightarrow \mathbb{R}^d$ as $M[\cdot] = M_f[M_p[\cdot]]$ with $M_f$ and $M_p$ the preconditioner for frequency and pixel, respectively.
We operate $M$ on the noise term $\mathbf{z}_t$ and adjust the gradient term to keep the steady-state distribution for the Langevin dynamics Eq.~\eqref{eq: reverse} as
\begin{align}\label{eq: accelerate2}
\begin{aligned}
        &\mathbf{x}_{t} = \mathbf{x}_{t-1} +  \frac{\epsilon_t^2}{2}M_p[M_f^{2}[M_p[\bigtriangledown_{\mathbf{x}}\log p_{\mathbf{x}^{\ast}}(\mathbf{x}_{t-1})]]] \\
        &+ \epsilon_t M_f[M_p[\mathbf{z}_t]], 1\leq t \leq T.
\end{aligned}
\end{align}
Similarly, we preserve the final-state distribution for the reverse diffusion Eq.~\eqref{eq:reverse_diffusion} as
\begin{align}\label{eq: accelerate3}
\begin{aligned}
    &\mathbf{x}_t= \mathbf{x}_{t-1} + \bar{g}_{t}^2 M_p[M_f^{2}[M_p[\bigtriangledown_{\mathbf{x}}\log \bar{p}_{t}(\mathbf{x}_{t-1})]]]+\\
    &\bar{g}_{t}M_f[M_p[\mathbf{z}_t]],1\leq t \leq T,
    \end{aligned}
\end{align}
We call this sampling procedure as {\em preconditioned diffusion sampling} (PDS), as summarized in Alg.~\ref{alg:pds}.

\begin{algorithm}[h]
  \caption{Preconditioned diffusion sampling (PDS)}
  \label{alg:pds}
\begin{algorithmic}
  \State {\bfseries Input:} 
  The frequency $M_f$ and pixel $M_p$ preconditioner, the sampling iterations $T$, the model $\mathbf{s}_{\theta}$.  
  \State Defining the preconditioning operator $M[\cdot]=M_f[M_p[\cdot]]$
  \State Drawing $\mathbf{x}_0 \sim \mathcal{N}(\mathbf{0},I)$
  \For{$t=1$ {\bfseries to} $T$} 
        \State \textbf{Predicting (reverse diffusion):} 
        \State Drawing $\mathbf{z}_t\sim\mathcal{N}(\mathbf{0},I)$
        \State $\mathbf{x}'_t\leftarrow \mathbf{x}_{t-1} + \bar{g}_{t}^2 M[M^{\mathrm{T}}[\mathbf{s}_{\theta}(\mathbf{x}_{t-1},\bar{g}_{t})]] +\bar{g}_{t}M[\mathbf{z}_t]$
        \State \textbf{Correcting (Langevin dynamics):} 
        \State Drawing $\mathbf{z}_t\sim\mathcal{N}(\mathbf{0},I)$
          \State $\mathbf{x}_{t} \leftarrow \mathbf{x}'_{t} + \frac{\epsilon_t^2}{2}M[M^{\mathrm{T}}[\mathbf{s}_{\theta}(\mathbf{x}'_{t},\epsilon_t) ]]+\epsilon_tM[\mathbf{z}_t]$
  \EndFor
  \State {\bfseries Output:} $\mathbf{x}_T$
\end{algorithmic}
\end{algorithm}

\subsection{Constructing image preconditioners}\label{sec:filter}
In this section, we design the frequency $M_f$ and pixel $M_p$ preconditioners. We denote an {image training dataset as $\mathcal{D} = \{\mathbf{x}^{(i)}\in\mathbb{R}^{C\times H\times W},i=1,\dots,N\}$}, where $C,H$ and $W$ are the channel number, height and width of images, respectively.

\subsubsection{Frequency preconditioner $M_f$}\label{sec:freq_filter}
Using the frequency statistics of the training set, we design the frequency preconditioner $M_p$ as
\begin{align}\label{eq:stats1}
    M_f[\mathbf{x}] = F^{-1}[F[\mathbf{x}] \bullet R_f],\quad\forall \mathbf{x}\in\mathbb{R}^{C\times W\times H}
\end{align}
where $F$ is Discrete Fourier Transform, $\bullet $ the element-wise division and the frequency mask $R_f\in\mathbb{R}^{C\times W\times H}$ defined as
\begin{align}\label{eq:freq_stats_R_f}
    R_f =  \log\big(\mathbb{E}_{\mathbf{x}\sim p^\ast(\mathbf{x})}\left [ F[\mathbf{x}]\odot \overline{F[\mathbf{x}]} \right]+1\big)
\end{align}
where $\odot $ is the element-wise multiplication. For better stability, we further normalize the mask as
\begin{equation}\label{eq:stats2}
   R_f (c,w,h) \leftarrow \frac{1}{\alpha}(\frac{R_f (c,w,h)}{\max_{c,w,h} R_f (c,w,h) }+\alpha-1),
\end{equation}
where $\alpha\in[1,+\infty)$ is the normalization parameter. The intuition is that the mask $M_f$ shrinks the frequency coordinates with larger scales more than the others so that the scales of all coordinates can become similar. This mitigates the ill-conditioned issue in the frequency space, as mentioned earlier.

\subsubsection{Pixel preconditioner $M_p$}\label{sec:pixel_filter}
With the pixel statistics, we similarly design the pixel preconditioner $M_p$ that operates in the pixel space as follows:
\begin{align*}
   M_p[\mathbf{x}] = \mathbf{x}\bullet R_p,\quad\forall \mathbf{x}\in\mathbb{R}^{C\times W\times H}
\end{align*}
with the pixel mask $R_p\in\mathbb{R}^{C\times W\times H}$ defined as
\begin{align}\label{eq:pixel_stats_R_p}
    R_p = \log\Big(\mathbb{E}_{\mathbf{x}\sim p^\ast(\mathbf{x})} \left[\mathbf{x} \odot \mathbf{x}  \right]+1\Big).
\end{align}
We also apply normalization as
\begin{align}\label{eq:r_p_soft}
  R_p(c,h,w) \leftarrow  \frac{1}{\alpha}\Big(\frac{R_p(c,h,w)}{\max_{c,h,w}R_p(c,h,w)}+\alpha-1 \Big)
\end{align}
with the normalization parameter $\alpha$ same as in Eq.~\eqref{eq:stats2}. 
For the tasks without clear pixel structure characteristics, we do not apply this pixel preconditioner by setting all the elements of $R_p$ to $1$ (\ie, an identity operator).

\subsubsection{Computational complexity}\label{sec:complexity}
{Only a small set of images is needed} to 
estimate the statistics for Eq.~\eqref{eq:freq_stats_R_f} and Eq.~\eqref{eq:pixel_stats_R_p}.
In practice, we use $200$ random images with marginal cost.
For the computational complexity of PDS during sampling, the major overhead is from FFT~\citep{brigham1988fast} and its inverse with the complexity of $O(CHW(\log H+\log W))=O(d(\log H+\log W))$.
This is neglectable w.r.t the original cost (see Table~\ref{tab: time}).

%% file: sections/6_further.tex
\section{Further Analysis}

\subsection{Gradient descent perspective}
We can understand the preconditioning {of} the Langevin dynamics from the gradient descent perspective. Given the target data distribution $p_{\mathbf{x}^\ast}$, we define the energy function ({\it a.k.a.} surprise in cognitive science~\citep{friston2010free}) as $E(\mathbf{x}) = -\log p_{\mathbf{x}^\ast} (\mathbf{x}),\quad \forall \mathbf{x}\in\mathbb{R}^d$. Substituting the definition of $E$ into Eq.~\eqref{eq: reverse}, we have
\begin{align*}
     \mathbf{x}_{t} = \mathbf{x}_{t-1} -  \frac{\epsilon_t^2}{2}\bigtriangledown_{\mathbf{x}}E(\mathbf{x}_{t-1}) + \epsilon_t \mathbf{z}_t, 1\leq t \leq T.
\end{align*}
Clearly, the Langevin dynamics is implementing the gradient descent algorithm for minimizing the energy of the sample point $\mathbf{x}_t$, along with an additive noise term $\mathbf{z}_t$. This optimizer is called stochastic gradient Langevin dynamics (SGLD)~\citep{welling2011bayesian}. Unlike stochastic gradient descent, the randomness of SGLD comes from the noise $\mathbf{z}_t$, instead of drawing from the training set. Besides, SGMs sequentially reduce the energy ($-\log p_{\mathbf{x}^{\ast}}(\mathbf{x})$) of a sample $\mathbf{x}_t$ via SGLD.

\subsection{Skew-symmetric operator}
In the gradient descent perspective,  $S\bigtriangledown_{\mathbf{x}}E(\mathbf{x}_{t-1})dt$ of Eq.~\eqref{eq: Langevin_M} is known as {the} solenoidal term, since it corresponds to the direction along which the energy does not change (\ie, $ \langle \bigtriangledown_{\mathbf{x}}E(\mathbf{x}_{t-1}), S\bigtriangledown_{\mathbf{x}}E(\mathbf{x}_{t-1})\rangle=0$). Seemingly it imposes no effect on acceleration. Several studies \citep{hwang2005accelerating,ottobre2016markov,rey2015irreversible,lelievre2013optimal} suggest that it can help reach the steady state distribution faster.

For completeness, we discuss the possibility of using {the} skew-symmetric operator $S$ for accelerating the Langevin dynamics.
According to~\citep{hwang2005accelerating}, under the regularity condition that $\left | \mathbf{x}_t\right |$ does not reach the infinity in a finite time, the convergence of an {\em autonomous} (the right side of the equation does not contain time explicitly) Langevin dynamics
\begin{align*}
    d\mathbf{x}_t = \frac{\epsilon_t^2}{2}\bigtriangledown_{\mathbf{x}_t}\log p_{\mathbf{x}^{\ast}}(\mathbf{x}_t)dt + \epsilon_t d\mathbf{w}_t,
\end{align*}
where $\epsilon_t=\epsilon$ is a constant, can be accelerated as
\begin{align*}
    d\mathbf{x}_t = \frac{\epsilon_t^2}{2}\bigtriangledown_{\mathbf{x}_t}\log p_{\mathbf{x}^{\ast}}(\mathbf{x}_t)dt +C(\mathbf{x}_t)dt  + \epsilon_t d\mathbf{w}_t,
\end{align*}
where the vector field $C(\mathbf{x})\in \mathbb{R}^d \rightarrow \mathbb{R}^d$ satisfies
\begin{align*}
    \bigtriangledown_{\mathbf{x}}\cdot(\frac{C(\mathbf{x})} {p_{\mathbf{x}^{\ast}}(\mathbf{x})}) = 0.
\end{align*}
$C(\mathbf{x}) =S\bigtriangledown_{\mathbf{x}}\log p_{\mathbf{x}^{\ast}}(\mathbf{x})$ satisfies the above condition.
However, the Langevin dynamics of existing SGMs is typically {\em not autonomous}, as the step size $\epsilon_t$ varies across time. Despite this, it is still worth investigating the effect of $S$ (see Sec.~\ref{sec: exp}).

\subsection{Parameter selection}\label{sec:theory_relation}
We investigate how to set properly $\alpha$ of Eq.~\eqref{eq:r_p_soft} and Eq.~\eqref{eq:stats2}.
Increasing $\alpha$ will weaken the preconditioning effect; When $\alpha \rightarrow +\infty$, $M_f$ and $M_p$ both become the identity operator.
Empirically, we find that for higher iterations $T$, higher $\alpha$ is preferred (Table~\ref{tab:pds_para}).
We aim to quantify the relation between $\alpha$ and $T$. 
By Thm.~\ref{thm:unchange2}, the preconditioned reverse diffusion (Eq.~\eqref{eq: accelerate3}) reverses the transformed forward process (Eq.~\eqref{eq:forward_M}), instead of the original one (Eq.~\eqref{eq:forward_diffusion_con}) used in model training. 
The deviation between Eq.~\eqref{eq:forward_M} and Eq.~\eqref{eq:forward_diffusion_con}
brings {a} challenge for the model to fit the score function, as the noisy input distribution deviates from what is exposed to model training at each step.
We assume the error per iteration grows linearly with this deviation, proportional to
\begin{align*}
    \left \| M_p[M_f[\cdot]]-I[\cdot]\right\|,
\end{align*}
where $I$ is the identity operator, and $\left \| \cdot\right\|$ is the $l_2$ norm of the linear operators in $\mathbb{R}^{d}\rightarrow \mathbb{R}^{d}$. 

With Eq.~\eqref{eq:stats2} and \eqref{eq:r_p_soft},
in case of large $\alpha$ we have 
\begin{align}\label{eq:q:expansion2}
    M_f(\alpha) \sim I + O(\frac{1}{\alpha})I,
    \;\;\;
    M_p(\alpha) \sim I + O(\frac{1}{\alpha})I.
\end{align}
By design, the elements of $R_f$ (Eq.~\eqref{eq:freq_stats_R_f}) and $R_p$ (Eq.~\eqref{eq:pixel_stats_R_p}) vary moderately.
Thus the accumulated error of the whole sampling process is $O(T\left \| M_p[M_f[\cdot]]-I[\cdot]\right\|)$.
With Eq.~\eqref{eq:q:expansion2}, this error scale can be written
\begin{align*}
    T\sqrt{d}[(1+O(\frac{1}{\alpha}))(1+O(\frac{1}{\alpha})) -1 ].
\end{align*}
Omit high-order terms, it could be simplified as
\begin{align}\label{eq:relation_1}
        T\sqrt{d}[(1+O(\frac{1}{\alpha}))(1+O(\frac{1}{\alpha})) -1 ] = O(1).
\end{align}
This implies that there exists constants $c_1,c_2,a$ and $b$  ($c_1,c_2,a\neq 0$) independent from $\alpha$ and $T$ such that
\begin{align}\label{eq:linear_1}
    (1+\frac{c_1}{\alpha})(1+\frac{c_2}{\alpha}) -1  = a\frac{1}{T}+b.
\end{align}
In case of only applying the frequency preconditioner $M_f$, Eq.~\eqref{eq:relation_1} becomes 
\begin{align}\label{eq:relation_2}
        T\sqrt{d}[1+O(\frac{1}{\alpha}) -1 ] = O(1),
\end{align}
and the corresponding linear relation is
\begin{align}\label{eq:linear_2}
    (1+\frac{c_1}{\alpha}) -1  = a\frac{1}{T}+b.
\end{align}
After quantifying this relation by searching the optimal $\alpha$ for only a couple of different $T$, we can directly estimate $\alpha$ for any other $T$ without manual search.

%% file: sections/7_exp.tex
\section{Experiments}\label{sec: exp}
In this section, we first show that PDS can accelerate the sampling process of state-of-the-art SGMs while maintaining the sampling quality and {providing} parameter analysis.
We use CIFAR-10~\citep{krizhevsky2009learning}, CelebA~\citep{liu2015faceattributes} at a resolution of $64\times64$, LSUN (bedroom and church)~\citep{yu2015lsun}  at a resolution of $256\times256$, and FFHQ~\citep{karras2019style} at a resolution of $1024\times1024$.
Note that for all these datasets, the image height and width are identical, \ie, $H=W$. The CelebA and FFHQ datasets are human facial images, with strong pixel-wise structural characteristics. {Hence, we apply both frequency and pixel filters to these two datasets. For other datasets, we only apply the frequency preconditioner.}
We adopt the same preprocessing as~\citep{song2020score}. We evaluate PDS on the state-of-art SGM \ie, NCSN++~\citep{song2020score}.
For a fair comparison, we only replace the vanilla sampling method (Alg.~\ref{alg:original_sample}) with our PDS (Alg.~\ref{alg:pds}) whilst keeping the remaining unchanged during inference for a variety of sampling iterations, without model retraining and other alteration.
Whenever reducing the iterations from the original one $T$ to $T/c$ with a factor $c$ we simultaneously expand the step size from the original one $\epsilon_t$ to $\sqrt{c}\epsilon_t$ for $t=1,\ldots,T$ for a consistent accumulative update of the score function and consistent variance of the added noises.
We use the public codebase of~\citep{song2020score}. We use the released checkpoints for all SGMs except NCSN++ on CelebA ($64\times64$), for which there is no released checkpoint, and we train by ourselves with the released codes instead.
See Appendix for the details of implementation.

\subsection{Evaluating on sampling quality}
We show that PDS can accelerate the sampling process of SGMs (NCSN++) for generating samples on datasets across various resolutions.
See Table~\ref{tab:pds_para} the setting of parameter $\alpha$ for different iterations $T$ by our linear relation (Sec. \ref{sec:theory_relation}).

\begin{table}[h]
\caption{Parameters $\alpha$ of PDS used for constructing preconditioner on NCSN++~\citep{song2020score} under different iterations $T$.
}\label{tab:pds_para}
\centering
\setlength{\tabcolsep}{1mm}{
\begin{tabular}{ccc}
\toprule[1.5pt]
\backslashbox[45mm]{Iterations $T$}{Dataset} & CIFAR-10  & CelebA ($64\times64$) \\ \midrule
1000                 & 50                          & 300                  \\
400                  & 25                          & 40                   \\
200                  & 12                          & 15                   \\
100                  & 5                           & 6                    \\
50                   & 1.8                         & 2.5                  \\ \bottomrule[1.5pt]
\end{tabular}}
\end{table}

\noindent{\bf CIFAR-10 and CelebA.}
In Table~\ref{tab:cifar} and Table~\ref{tab:celeba}, we evaluate the PDS with NCSN++ on CIFAR-10 ($32\times 32$) and CelebA ($64\times 64$) under various iterations $T$.
We provide visual samples on CIFAR-10 and CelebA in Appendix.
The noise schedule defines how noise is progressively added to the data during the forward process. Two commonly used noise schedules are linear and cosine. Unless specified otherwise, the methods presented in Table~\ref{tab:cifar}, Table~\ref{tab:celeba}, and Table~\ref{tab:fid_com} employed a linear noise schedule. Since our baseline, NCSN++~\citep{song2020score}, employs a linear schedule, we also adopt a linear schedule for the PDS evaluation.

Although achieving on-par performance with DDPMs given high iterations, NCSN++ degrades more significantly as the number of iterations decreases. 
Specifically, NCSN++ performs the worst among all methods at both 50 and 100 iterations. Our PDS addresses this limitation while further enhancing performance of the NCSN++ at high iteration. As demonstrated in Table~\ref{tab:cifar} and Table~\ref{tab:celeba}, our PDS achieves the best FID scores at 100 iterations on both the CIFAR-10 and CelebA datasets, while also delivering competitive results at 50 iterations. The reason PNDM~\citep{liu2021pseudo} outperforms our PDS at the 50 iterations is that our base method NCSN++ suffers from significant numerical instability under low-iteration conditions, which cannot be entirely resolved by our PDS. Furthermore, as shown in Table~\ref{tab:fid_com}, our PDS achieves an FID score of 1.99 on CIFAR-10, surpassing prior competitors.

\begin{table*}[h]
\caption{{FID scores
under different iterations  $T$ on CIFAR-10.}}\label{tab:cifar}
\centering
\begin{tabular}{lrrrrr}
\toprule[1.5pt]
\multicolumn{1}{c}{Iterations $T$}   & \multicolumn{1}{c}{50} & \multicolumn{1}{c}{100} & \multicolumn{1}{c}{200} & \multicolumn{1}{c}{400} & \multicolumn{1}{c}{1000} \\ \midrule
NCSN++~\citep{song2020score}          & 306.91                 & 29.39                   & 4.35                    & 2.40                    & 2.23                     \\
\textbf{NCSN++ w/ PDS}               & 4.90                   & \textbf{2.90}           & \textbf{2.53}           & \textbf{2.19}           & \textbf{1.99}            \\
DDPM~\citep{DBLP:conf/nips/HoJA20}    & 66.28                  & 31.36                   & 12.96                   & 4.86                    & 3.17                     \\
DDIM~\citep{song2020denoising}        & 7.74                   & 6.08                    & 5.07                    & 4.61                    & 4.13                     \\
{Analytic-DDPM ({\it linear})}~\citep{bao2022analytic} & 7.25                   & 5.40                    & 4.01                    & 3.62                    & 4.03                     \\
{Analytic-DDPM ({\it cosine})}~\citep{bao2022analytic} & 5.50                   & 4.45                    & 4.04                   & 3.96                    & 4.31                     \\
{Analytic-DDIM ({\it linear})}~\citep{bao2022analytic} & 4.04                   & 3.55                    & 3.39                    & 3.50                    & 3.74                     \\
{Analytic-DDIM ({\it cosine})}~\citep{bao2022analytic} & 6.02                   & 4.88                    & 4.92                    & 5.00                    & 4.66                     \\
{PNDM ({\it linear})}~\citep{liu2021pseudo}   & {3.95}                 & 3.72                    & 3.48                    & 3.57                    & 3.70                     \\ 
{PNDM ({\it cosine})}~\citep{liu2021pseudo}   &  \textbf{3.68}         & 3.53                    & 3.50                    & 3.30                    & 3.26                 \\ \bottomrule[1.5pt]
\end{tabular}
\end{table*}

\begin{table*}[h]
\caption{{FID scores under different iterations $T$ on CelebA ($64\times 64$)}.
}\label{tab:celeba}
\centering
\begin{tabular}{lrrrrr}
\toprule[1.5pt]
\multicolumn{1}{c}{Iterations $T$}   & \multicolumn{1}{c}{50} & \multicolumn{1}{c}{100} & \multicolumn{1}{c}{200} & \multicolumn{1}{c}{400} & \multicolumn{1}{c}{1000} \\ \midrule
NCSN++~\citep{song2020score}          & 438.84                 & 433.98                  & 292.51                  & 7.09                    & 3.33                     \\
\textbf{NCSN++ w/ PDS}               & 13.48                  & \textbf{3.49}           & 3.21           & \textbf{2.95}           & 3.25                     \\
DDPM~\citep{DBLP:conf/nips/HoJA20}    & 115.69                 & 25.65                   & 9.72                    & 3.95                    & 3.16                     \\
DDIM~\citep{song2020denoising}        & 9.33                   & 6.60                    & 4.96                    & 4.15                    & 3.40                     \\
Analytic-DDPM~\citep{bao2022analytic} & 11.23                  & 8.08                    & 6.51                    & 5.87                    & 5.21                     \\
Analytic-DDIM~\citep{bao2022analytic} & 6.13                   & 4.29                    & 3.46                    & 3.38                    & 3.13                     \\
{PNDM}~\citep{liu2021pseudo}     & \textbf{5.69}          & 4.03                    & \textbf{{3.04}}                   & 3.07                    & \textbf{2.99}            \\ 
\bottomrule[1.5pt]
\end{tabular}
\end{table*}

\begin{table*}[h]
\centering
\caption{{Comparison of Inception scores between NCSN++ and NCSN++ with our PDS across varying iterations $T$ on CIFAR-10 and CelebA ($64 \times 64$}).
}\label{tab:inception}
\begin{tabular}{clrrrrr}
\toprule[1.5pt]
\multicolumn{1}{r}{Dataset} & \multicolumn{1}{c}{Methods} & \multicolumn{1}{c}{50} & \multicolumn{1}{c}{100} & \multicolumn{1}{c}{200} & \multicolumn{1}{c}{400} & \multicolumn{1}{c}{1000} \\ \midrule
\multirow{2}{*}{CIFAR-10}   & NCSN++~\citep{song2020score}                      & 1.54                   & 6.99                    & 9.29                    & 9.53                    & 9.73                     \\
                            & {\bf NCSN++ w/ PDS}               & \textbf{9.53}          & \textbf{9.36}           & \textbf{9.42}           & \textbf{9.56}           & \textbf{9.81}            \\ \hline
\multirow{2}{*}{CelebA}     & NCSN++                      & 1.14                   & 1.15                    & \textbf{2.43}           & \textbf{2.54}           & \textbf{2.23}            \\
                            & {\bf NCSN++ w/ PDS}               & \textbf{2.54}          & \textbf{2.47}           & 2.41                    & 2.28                    & 2.22                     \\ \bottomrule[1.5pt]
\end{tabular}
\end{table*}

\begin{table*}[h]
\caption{FID score under various iterations $T$ on ImageNet ($64\times64$).}\label{tab:image_fid}
\centering
\begin{tabular}{lrrrr}
\toprule[1.5pt]
\multicolumn{1}{c}{Iterations $T$}   & \multicolumn{1}{c}{100} & \multicolumn{1}{c}{200} & \multicolumn{1}{c}{400} & \multicolumn{1}{c}{1000} \\ \midrule
NCSN++~\citep{song2020score}          & 422.54                  & 134.17                  & 33.48                   & 18.93                    \\
\textbf{NCSN++ w/ PDS}               & 19.1                    & \textbf{17.08}          & \textbf{16.15}          & \textbf{15.63}           \\
DDPM~\citep{DBLP:conf/nips/HoJA20}    & 45.04                   & 28.39                   & 21.38                   & 17.58                    \\
DDIM~\citep{song2020denoising}        & 18.09                   & 17.84                   & 17.74                   & 17.73                    \\
Analytic-DDPM~\citep{bao2022analytic} & 18.80                   & 17.16                   & 16.40                   & 16.14                    \\
Analytic-DDIM~\citep{bao2022analytic} & \textbf{17.73}          & 17.49                   & 17.44                   & 17.57        \\\bottomrule[1.5pt]           
\end{tabular}
\end{table*}

\begin{table*}[h]
\caption{FID score under various iterations $T$ on COCO.}\label{tab:coco_fid}
\centering
\begin{tabular}{lrrrrrrr}
\toprule[1.5pt]
\multicolumn{1}{c}{Iterations $T$}             & \multicolumn{1}{c}{2000} & \multicolumn{1}{c}{1000} & \multicolumn{1}{c}{500} & \multicolumn{1}{c}{400} & \multicolumn{1}{c}{200} & \multicolumn{1}{c}{100} \\ \midrule
  NCSN++        & 12.41          & 15.82             & 37.71     & 65.39          & 339.59                  & 464.48                           \\
 {\bf NCSN++ w/ PDS} & \textbf{11.76}     & \textbf{12.77}             & \textbf{13.51}       & \textbf{13.63}            & \textbf{15.94}           & \textbf{19.57}                  \\ \bottomrule[1.5pt]
\end{tabular}
\end{table*}

\begin{table}[h]
\caption{{Best FID comparison on CIFAR-10.}}\label{tab:fid_com}
\centering
\setlength{\tabcolsep}{1.5mm}{
\begin{tabular}{llc}
\toprule[1.5pt]
Type &Method        & FID           \\ \midrule
\multirow{2}{*}{GAN}& BigGAN~\citep{brock2018large} & 14.73\\
& StyleGAN2-ADA~\citep{kang2021rebooting}&2.42
\\ \midrule
\multirow{5}{*}{DDPM}&DDPM~\citep{DBLP:conf/nips/HoJA20}          & 3.17          \\
&DDIM~\cite{song2020denoising}          & 4.13          \\
&{Analytic-DDPM ({\it linear})}~\citep{bao2022analytic} & 3.63          \\
&{Analytic-DDPM ({\it cosine})}~\citep{bao2022analytic} & 4.04          \\
&{Analytic-DDIM ({\it linear})}~\citep{bao2022analytic} & 3.39          \\
&{Analytic-DDIM ({\it cosine})}~\citep{bao2022analytic} & 4.66          \\
&{PNDM ({\it linear})}~\citep{liu2021pseudo}          &  3.48        \\
&{PNDM ({\it cosine})}~\citep{liu2021pseudo}          &  3.26        \\
 &DDPM++~\citep{song2020score}        & 2.91          \\
\midrule
\multirow{6}{*}{SGM}&NCSN++~\citep{song2020score}        & 2.23          \\
 &CLD-SGM~\citep{dockhorn2021score}       & 2.23          \\
 &LSGM~\citep{vahdat2021score}          & 2.10          \\
 &INDM~\citep{kim2022maximum}          & 2.28          \\
\rowcolor{Gray}
&NCSN++ w/ PDS & \textbf{1.99} \\ \bottomrule[1.5pt]
\end{tabular}}
\end{table}

\noindent{\bf LSUN~\citep{yu2015lsun}.} 
We test PDS with NCSN++ to generate bedroom and church images at a resolution of $256\times 256$. As shown in Fig.~\ref{fig: lsun_ncsnpp}, our PDS can prevent outputs from being ruined by heavy noises under low iterations. See more visualization in {the} Appendix.

\noindent{\bf FFHQ~\citep{karras2019style}.} 
We test PDS to generate facial images at a resolution of $1024\times 1024$. Similarly, under acceleration the original sampling method suffers from heavy noises and fails to produce recognizable human faces. Our PDS maintains the sampling quality with 66 iterations (Fig.~\ref{fig:ffhq} and Fig.~\ref{fig:ffhq_hq}). 

\noindent {\bf ImageNet~\citep{deng2009imagenet} and COCO~\citep{lin2014microsoft}. } 
We further investigate the effect of PDS on NCSN++ in generating image samples on ImageNet and COCO respectively. As the first ImageNet and COCO evaluation for VE-SDE models, we train an NCSN++ by ourselves at the resolution of 64$\times$64 as the baseline.
As shown in Table~\ref{tab:image_fid}, our PDS yields consistently acceleration effect across a variety of sampling iterations and produces comparable performance compared with NCSN++, DDPM~\citep{DBLP:conf/nips/HoJA20}, DDIM~\citep{song2020denoising}, Analytic-DDPM~\citep{bao2022analytic} and Analytic-DDIM~\citep{bao2022analytic}. Note that our PDS reaches the best FID score (15.63) among all methods. Since there is no available codebase of DDPM, DDIM, Analytic-DDPM, and Analytic-DDIM for COCO dataset generation, we compare our PDS with the baseline NCSN++. As shown in Tab.~\ref{tab:coco_fid}, our PDS also enhances the generation quality of NCSN++ across different sampling iterations on COCO. We showcase more samples in Fig.~\ref{fig:imagenet}-\ref{fig: ncsnpp_lsun} of the Appendix.

\begin{figure}[h]
\centering
\includegraphics[width=0.48\textwidth]{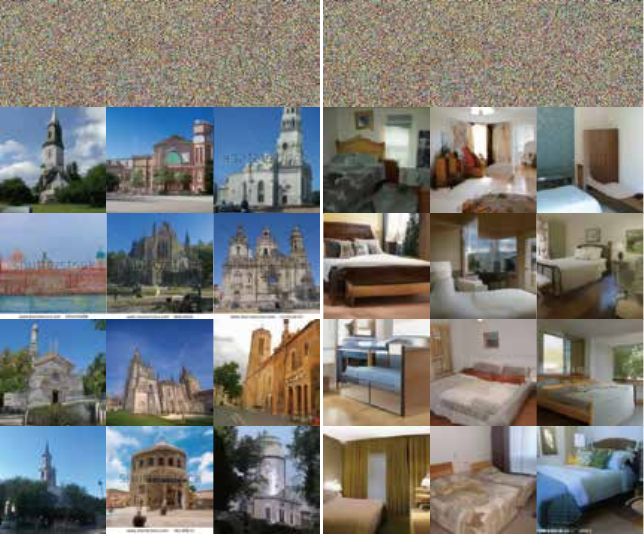}
        \caption{Sampling 
        on LSUN (church and bedroom) ($256 \times 256$). 
        \textbf{Top}: NCSN++~\citep{song2020score}
        \textbf{Bottom}: NCSN++ with PDS.
        Both use $80$ iterations and $\alpha=2.2$. More samples in Appendix.}
        \label{fig: lsun_ncsnpp}
\end{figure}

\subsection{Evaluation on running speed}
We evaluate the effect of running speed {by} using PDS.
In this test, we use one NVIDIA RTX 3090 GPU. We track the average wall-clock time of generating a batch of 8 images. As shown in Table \ref{tab: time}, PDS only introduces marginal extra time complexity per iteration, corroborating our complexity analysis in Sec.~\ref{sec:complexity}.
Importantly, PDS enables low-iteration generation for better efficiency. For example, for the high-resolution generation task on FFHQ ($1024\time1024$), PDS reduces the time cost of generating $8$ human facial samples from $1920$ to $68$ in seconds, \ie, {28}$\times$ speedup.

\begin{table}[h]
\centering
\caption{Comparing the running time \textit{per iteration} of generating a batch of $8$ images.
{\em Time unit:} Seconds.}
\label{tab: time}
\setlength{\tabcolsep}{1.2mm}{
\begin{tabular}{llll}
\toprule[1.5pt]
\multicolumn{1}{l}{Dataset} & \multicolumn{1}{c}{CelebA} & \multicolumn{1}{c}{LSUN}  & \multicolumn{1}{c}{FFHQ} \\ 
\multicolumn{1}{l}{Resolution} & \multicolumn{1}{c}{$64\times64$} & \multicolumn{1}{c}{$256\times256$}  & \multicolumn{1}{c}{$1024\times1024$} \\ \midrule
\multicolumn{1}{l}{NCSN++}     & \multicolumn{1}{c}{0.10}  & \multicolumn{1}{c}{0.51} & \multicolumn{1}{c}{0.96}     \\
\multicolumn{1}{l}{NCSN++ w/ PDS}    & \multicolumn{1}{c}{0.12}      & \multicolumn{1}{c}{0.54}  & \multicolumn{1}{c}{1.03}      \\
\bottomrule[1.5pt]
\end{tabular}}
\end{table}

\subsection{Experiments on other SGMs}
We evaluate the performance of PDS on two more SGMs, \ie, NCSN~\citep{song2019generative} and NCSNv2~\citep{DBLP:conf/nips/0011E20}.
Although the approach of constructing preconditioning operators in Sec.~\ref{sec:filter} works well for accelerating NCSN++~\citep{song2020score}, it is less effective for NCSN~\citep{song2019generative} and NCSNv2~\citep{DBLP:conf/nips/0011E20}. 
A possible reason is that the two models are less capable of exploiting the data statistics. To address this, we construct a simpler version of {the} frequency preconditioner as
\begin{align}\label{eq: freq_mask_a}
R_f(c,h,w) = 
\left\{\begin{matrix}
 1, \quad (h-H/2)^2+(w-W/2)^2 \leq 2 r^2\\
\lambda,\quad (h-H/2)^2+(w-W/2)^2 > 2 r^2
\end{matrix}\right.,
\end{align}
where $C$ is the channel number, $H$ is the height, and $W$ is the width of an image.
$1\leq c \leq C$, $1\leq h \leq H$ and $1 \leq w \leq W$.
The parameter $\lambda > 0 $ specifies the ratio for scaling the frequency coordinates located out of the area $\{(h-0.5H)^2+(w-0.5 W)^2 \leq 2 r^2 \}$. The radial range of the area is controlled by $r$.
See an illustration of the $R_f$ in Fig.~\ref{fig:freq_mask_a}. We do not normalize $R_f$ and directly use it for constructing the corresponding preconditioning operator
\begin{align}\label{eq: freq_mask_a2}
    M_f[\mathbf{x}] = F^{-1}[F[\mathbf{x}] \bullet R_f].
\end{align}
Note that NCSN and NCSNv2 have only {correctors} in the sampling process. We set $\lambda$ so that the operator $M_f$ balances the coordinates of low and high-frequency components of its input via shrinking the scale of low-frequency coordinates.
This is because that the scale of low-frequency coordinates is higher than that of high-frequency coordinates among natural images \citep{Schaaf1996ModellingTP}. As shown in Fig.~\ref{fig: mnist_sm}-\ref{fig: ncsnv2_tower_sm}, PDS can dramatically accelerate the sampling process of NCSN and NCSNv2.

\begin{figure}[h]
\centering
 \includegraphics[width=0.2\textwidth]{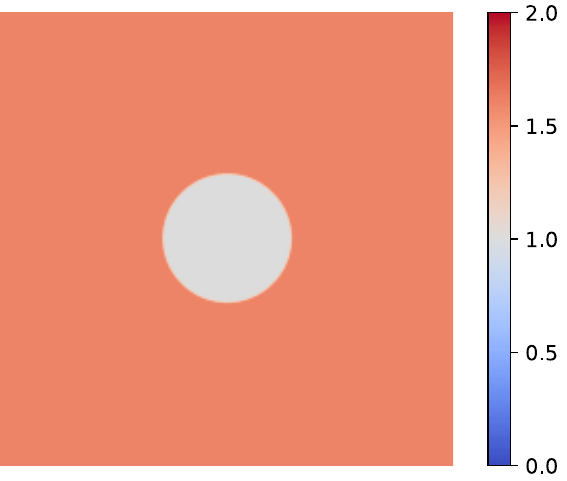}
    \caption{Examples of frequency preconditioning $R_f$ with $(r,\lambda) = (0.1H,1.6)$ in Eq.~\eqref{eq: freq_mask_a2}. 
    }
    \label{fig:freq_mask_a}
\end{figure}

\begin{figure}[h]
    \centerline{\includegraphics[width=0.48\textwidth]{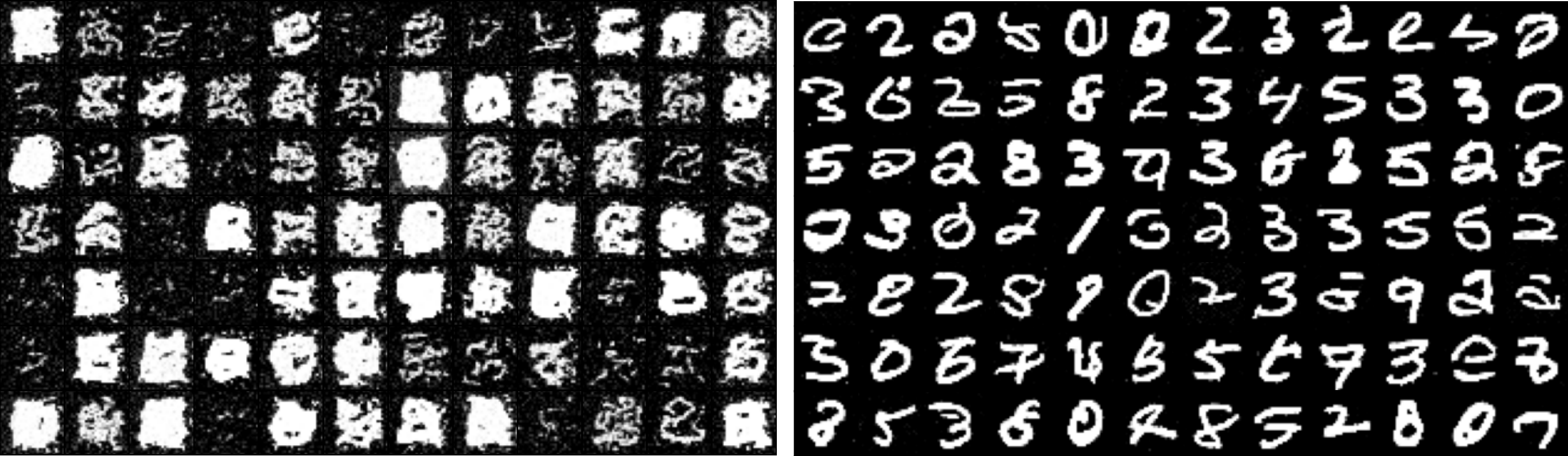}}
    \caption{Sampling using NCSN~\citep{song2019generative} on MNIST. \textbf{Left}: Results by the original sampling method. \textbf{Right}: Results by our PDS ($(r,\lambda)=(0.2H,1.6)$). Iterations: $T=20$.}\label{fig: mnist_sm}
\end{figure}

\begin{figure}[h]
    \centerline{\includegraphics[width=0.48\textwidth]{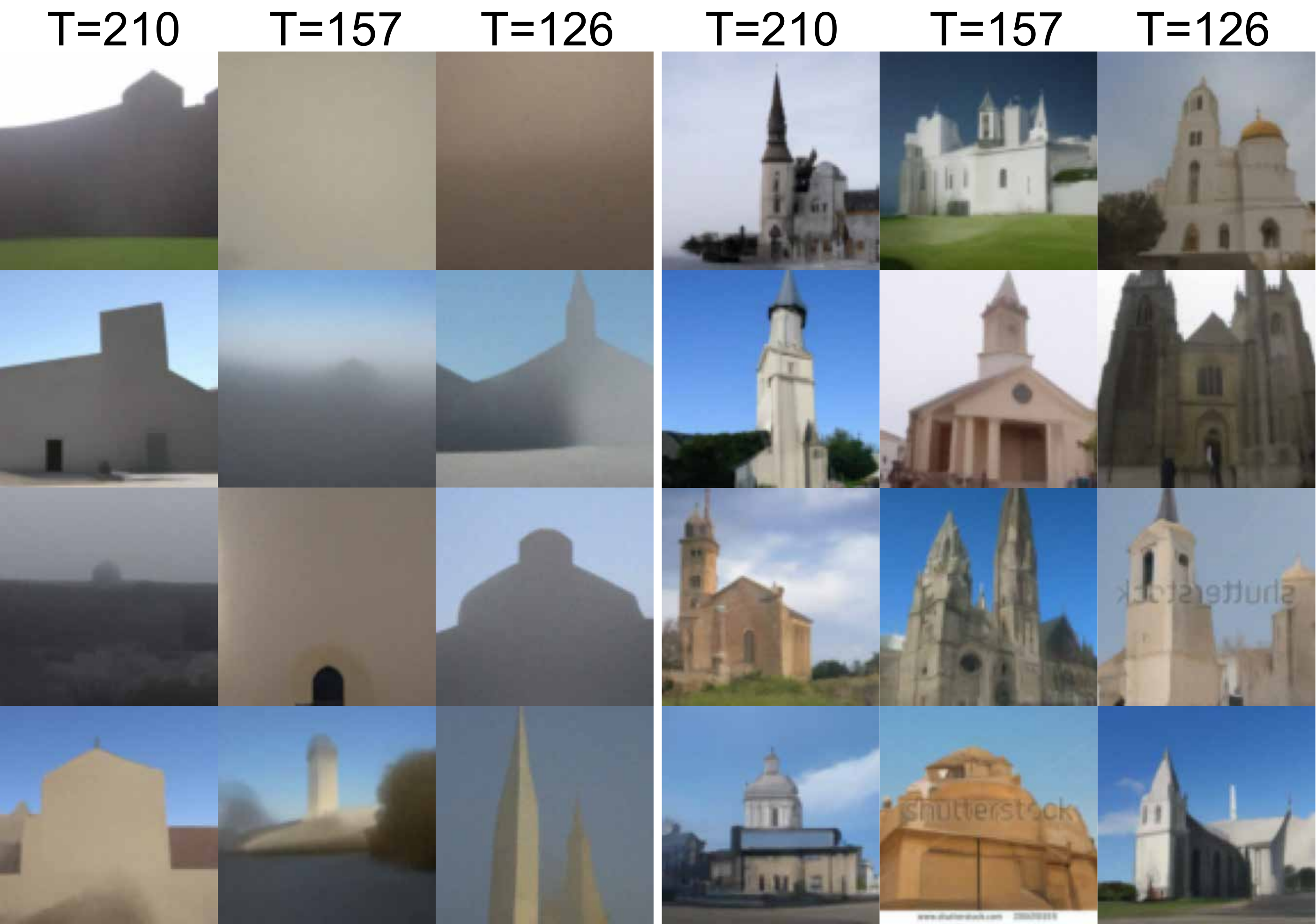}}
    \caption{Sampling using NCSNv2~\citep{DBLP:conf/nips/0011E20} on LSUN church ($96 \times 96$). \textbf{Left}: Results by the original sampling method with different sampling iterations. \textbf{Right}: Results by our PDS ($(r,\lambda)=(0.2H,1.1)$) with different iterations $T$.
        }      
        \label{fig: ncsnv2_church_sm}
\end{figure}

\begin{figure}[h]
    \centerline{\includegraphics[width=0.48\textwidth]{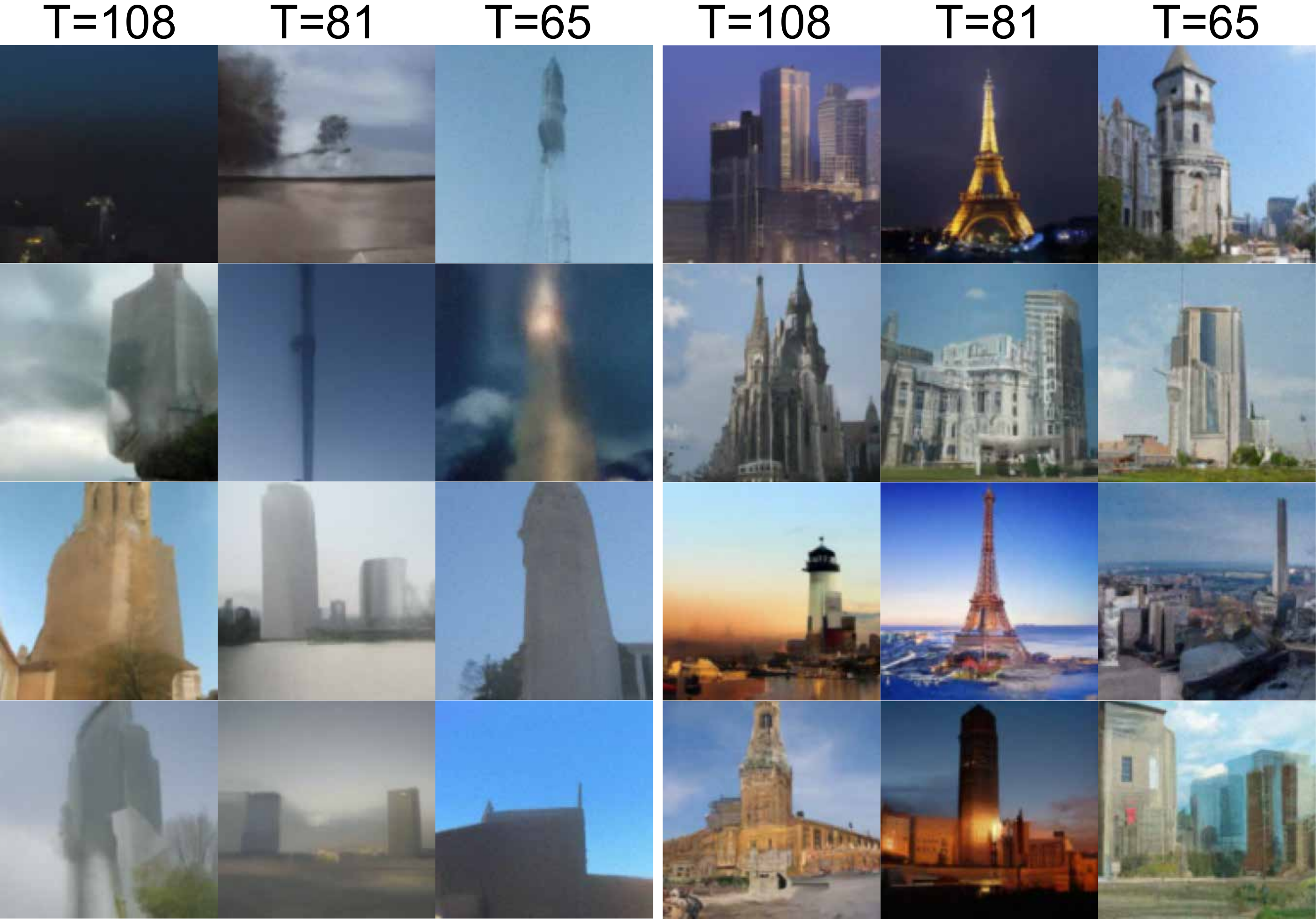}}
    \caption{Sampling using NCSNv2~\citep{DBLP:conf/nips/0011E20} on LSUN tower ($128 \times 128$). \textbf{Left}: Results by the original sampling method with different sampling iterations. \textbf{Right}: Results by our PDS ($(r,\lambda)=(0.2H,1.6)$) with different iterations $T$.
        }      
        \label{fig: ncsnv2_tower_sm}
\end{figure}

\subsection{Compared with subspace diffusion generative models}
The subspace diffusion generative models (SDGMs)~\citep{jing2022subspace} restrict the sampling process of the SGMs on a low-dimensional subspace for acceleration. This might help mitigate the issue of ill-conditioning. We thus compare our PDS with SDGMs in terms of the effect of alleviating ill-conditioning.

We conduct two separate experiments to evaluate their performance. In the first experiment, we investigate the performance of two methods in terms of varying iterations. As shown in Table~\ref{tab:sgdm}, our PDS achieves lower/better FID scores when compared to SDGMs subject to the same iterations.
It is important to highlight that SDGM benefits from a smaller per-iteration cost due to reduced dimensionality, allowing it to execute more iterations within the same runtime.
Hence we conduct a second experiment with a comparison taking into account the runtime and sample quality of the two methods, a different fair evaluation. 
As depicted in Fig.~\ref{fig:sdgm}, our PDS demonstrates the ability to generate higher-quality images (as indicated by a lower FID score) at the equivalent runtime cost when compared to SDGMs. In summary, our preconditioning strategy surpasses the dimension reduction idea of SDGM.

\begin{table}[h]
    \centering
    \caption{Comparison of our PDS and SDGM~\citep{jing2022subspace} on CIFAR-10 generation under varying iterations.}
    \label{tab:sgdm}
    \setlength{\tabcolsep}{1.5mm}{
    \begin{tabular}{ccccccc}
    \toprule[1.5pt]
      Iterations $T$&50    & 100  & 200 & 400  & 500 &1000\\\midrule
      PDS           &\textbf{4.90}  & \textbf{2.90} & \textbf{2.53}& \textbf{2.19} & \textbf{2.11}&\textbf{1.99}\\
      SDGM          &386.70& 64.05& 7.03& 3.03 & 2.70&2.17\\ \bottomrule[1.5pt]
    \end{tabular}}
\end{table}

\begin{figure}[h]
    \centering
    \includegraphics[width=0.4\textwidth]{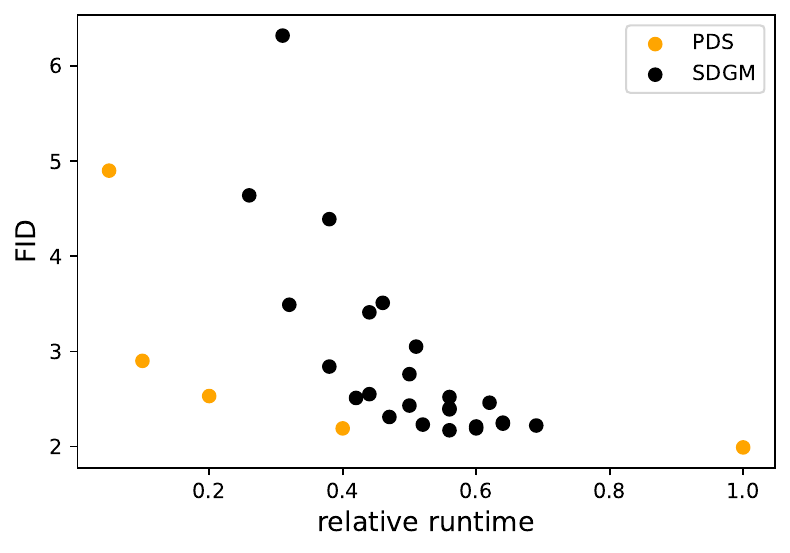}
    \caption{Comparison of our PDS and subspace diffusion
generative models (SDGM)~\citep{jing2022subspace} on CIFAR-10. We report the FID score and the relative runtimes with respect to the vanilla SGM.
    }
    \label{fig:sdgm}
\end{figure}

%% file: sections/8_further2.tex
\section{Further experiments}
\subsection{Effect of PDS for alleviating ill-conditioning}\label{sec:effect}
We empirically investigate the effect of PDS to reduce the ill-conditioning of the original generation process of SGM. To measure the degree of ill-conditioning, we define the following metrics: the coordinate variation and $V_{\mathrm{coo}}$ (variance across coordinates) and coordinate range $R_{\mathrm{coo}}$ (max coordinate value -min coordinate value). For the sample vector $\mathbf{x}\in \mathbb{R}^d$, the coordinate variation $V_{\mathrm{coo}}$ is define as $V_{\mathrm{coo}}=\sum_{i=1}^d (\mathbf{x}_i-\frac{1}{d}\sum_{j=1}^d \mathbf{x}_j)^2$, and the coordinate range $R_{\mathrm{coo}}$ is define as $R_{\mathrm{coo}}=\max_i \mathbf{x}_i - \min \mathbf{x}_i$. Both metrics indicate the gap of the scales in different coordinates, i.e., the degree of the ill-condition issue, with larger $R_{\mathrm{coo}}$ and $V_{\mathrm{coo}}$ indicating more severe ill-condition issue.
We add new experiments, in which we generate some samples through NCSN++ with and without PDS on CIFAR-10, CelebA 64x64, and ImageNet and evaluate the above two metrics. Both metrics are average over all samples and iteration time points for each dataset. As shown in Fig.~\ref{fig:effect}, PDS can reduce both coordinate variation and $V_{\mathrm{coo}}$ coordinate range $R_{\mathrm{coo}}$, hence {alleviating} the ill-condition issue.

\begin{figure}[h]
    \centering
    \includegraphics[width=0.48\textwidth]{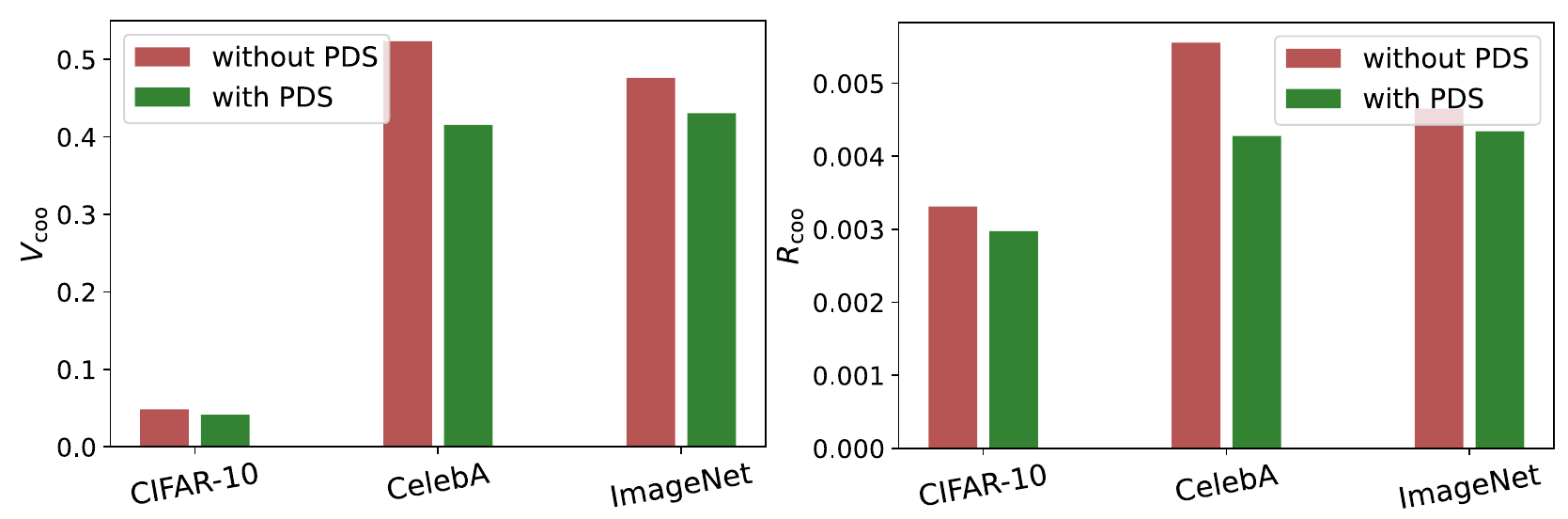}
    \caption{{The coordinate variation $V_{\mathrm{coo}}$ and coordinate range $R_{\mathrm{coo}}$ of NCSN++~\citep{song2020score} with PDS and without PDS for CIFAR-10, ImageNet and CelebA with iterations $T=100$.}}
    \label{fig:effect}
\end{figure}

\subsection{Relation between $\alpha$ and $T$}\label{sec:relation}
The linear relation between $\alpha$ and $T$, as derived in Sec.~\ref{sec:theory_relation}, can ease the estimation of $\alpha$ for different $T$.
To quantify this, we just need to search the optimal $\alpha$ for two different $T$. 
To simplify Eq.~\eqref{eq:linear_1} and~\eqref{eq:linear_2} we set $c_1=c_2=1$ and have
\begin{equation}
    (1+\frac{1}{\alpha})^2 -1  = a_2\frac{1}{T}+b_2, \;\; 
    (1+\frac{1}{\alpha}) -1  = a_1\frac{1}{T}+b_1.
\end{equation}
We test CIFAR-10 with only the frequency preconditioner and CelebA, ImageNet, and COCO ($64\times64$) with both. The results of linear regression are depicted in Fig.~\ref{fig:para_relation}. We obtain the coefficient of determination of 0.895 on CIFAR-10, 0.940 on CelebA($64\times64$), 0.992 on ImageNet ($64\times64$), and 0.804 respectively. This suggests strong linear relations, validating our analysis.
Note, when $(1+\frac{1}{\alpha})^2 -1 $ and $(1+\frac{1}{\alpha}) -1 $ are similar,
the ratio of the slope $a$ will be close to the square root of the dimension ratio.
For instance, the relation between CIFAR-10 and CelebA($64\times64$) fits as
\begin{align*}
    \frac{a_1}{a_2} = \frac{28.53}{50.67} = 0.56 \sim 0.5 = \sqrt{\frac{3 \times 32\times 32}{3 \times 64\times 64}}= \sqrt{\frac{d_1}{d_2}},
\end{align*}
where $a_1,a_2$ are the slopes, and $d_1,d_2$ are their dimensions for {CIFAR-10, CelebA ($64\times64$), and ImageNet ($64\times64$)}, respectively.

\begin{figure}[h]
    \centering
    \includegraphics[width=0.4\textwidth]{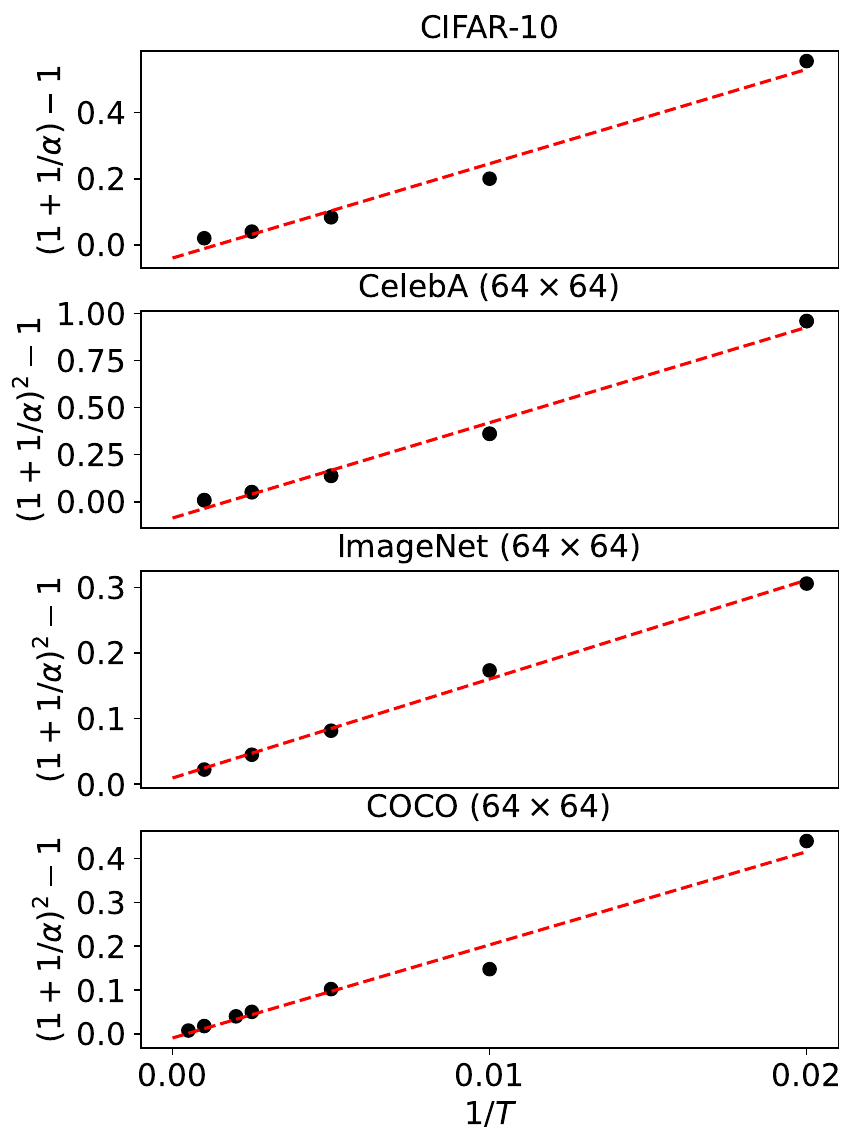}
    \caption{Relation between the iterations $T$ and the parameter $\alpha$. 
    All cases suggest a strong linear correlation.}\label{fig:para_relation}
\end{figure}

\subsection{Sensitivity of parameters}
We evaluate the parameter sensitivity of PDS. We perturb the optimal $\alpha$ (Table~\ref{tab:pds_para}) by $ \pm 10\%$. As shown in Table~\ref{pert_cifar}, the performance is not parameter-sensitive.

\begin{table}[h]
\centering
\caption{
Sensitivity analysis
on of the parameter($\alpha$) on the CIFAR-10.}\label{pert_cifar}
\begin{tabular}{cccc}
\toprule[1.5pt]
Iterations $T$     & before perturbation & +10\% & -10\% \\ \midrule
1000 & 1.99               & 2.00 & 2.01 \\
500  & 2.11               & 2.13 & 2.12 \\
400  & 2.19               & 2.20 & 2.19 \\
200  & 2.53               & 2.56 & 2.58 \\
100  & 2.90               & 3.16 & 3.10 \\
50   & 4.90               & 5.45 & 7.34 \\ \bottomrule[1.5pt]
\end{tabular}
\end{table}

\subsection{Skew-symmetric operator}
\begin{figure}[h]
\centering
    \begin{minipage}{1\textwidth}
    \textcolor{white}{--------}$\omega =1 $\textcolor{white}{-----------}$\omega = 10$\textcolor{white}{----------}$\omega= 100$\textcolor{white}{----------}$\omega = 1000$
    \end{minipage}

    \vspace{.4ex}
    \begin{minipage}{1\linewidth}
    \centering
    \centerline{\includegraphics[scale=0.13]{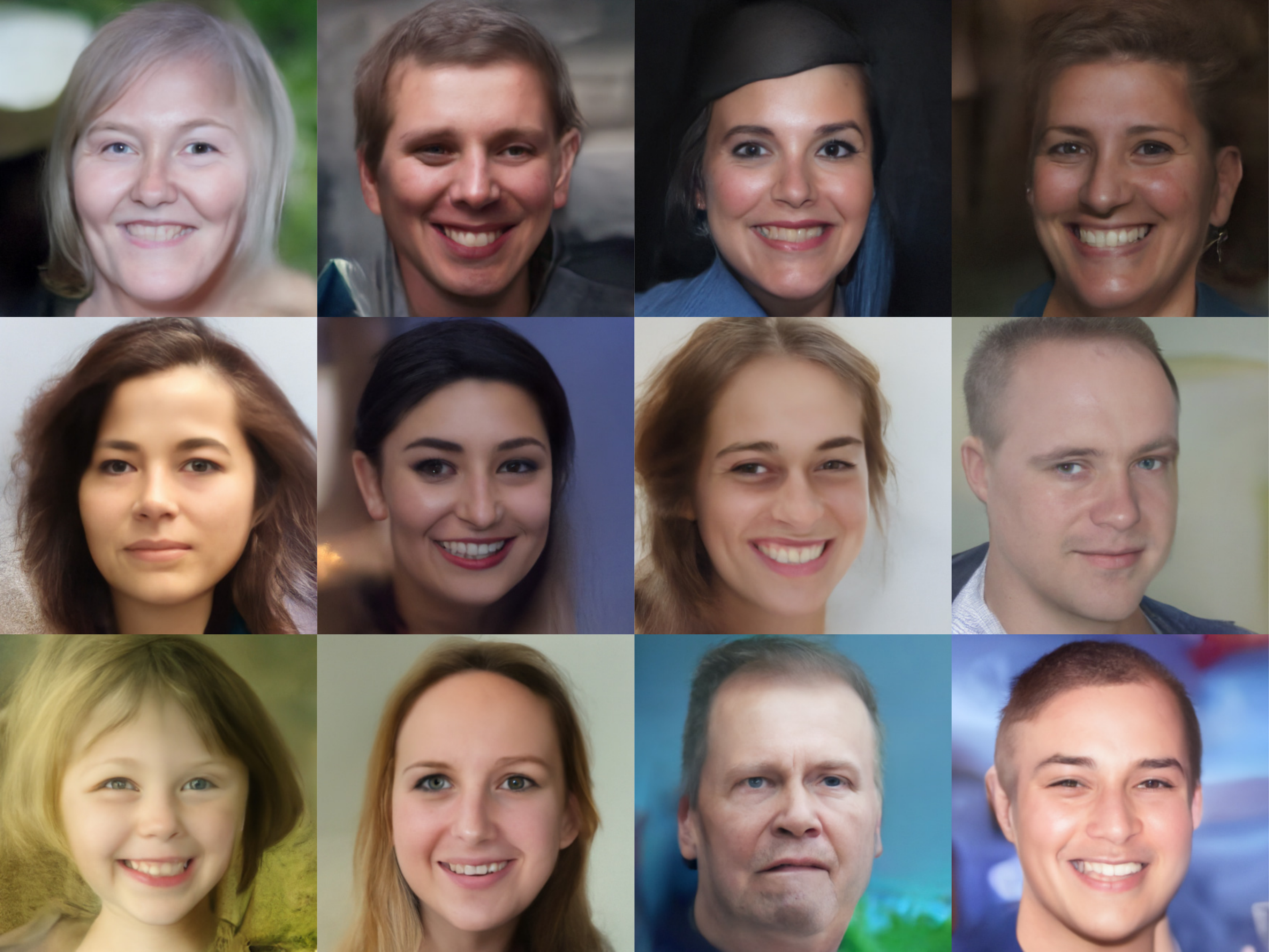}}
    \end{minipage}
        \caption{Samples produced by NCSN++~\citep{song2020score} w/ PDS on FFHQ ($1024$x$1024$)~\citep{karras2019style} with different solenoidal terms. Iterations: $66$. }
        \label{fig: solenoidal}
\end{figure}

We study the effect of the solenoidal term  $S\bigtriangledown_{\mathbf{x}} \log p_{\mathbf{x}^\ast}(\mathbf{x})$ to the Langevin dynamics. As proved in Thm.~\ref{thm: unchange}, as long as $S$ is skew-symmetric, it will not change the steady-state distribution of the original process. To verify this, we consider the process
\begin{align*}
        d\mathbf{x}_t =  \frac{\epsilon_{t}^2}{2}(MM^{\mathrm{T}}+\omega S)\bigtriangledown_{\mathbf{x}_t}\log p_{\mathbf{x}^{\ast}}(\mathbf{x}_t)dt  + \epsilon_tM d\mathbf{w}_t,
\end{align*}
where $\omega$ controls the weight of $S$. We experiment with $S[\cdot] = \mathrm{Re}[F[\cdot] - F^{\mathrm{T}}[\cdot]]$ for being skew-symmetric
and $\omega \in [1,1000]$. Fig.~\ref{fig: solenoidal} shows that $\omega$ has no impact {on} the generation quality. This verifies that $S$ does not change the original steady-state distribution.
We further perform similar tests with different iterations and other different skew-symmetric operator
\begin{align}
    &S_1= P_{1,1} - P_{1,1}^{\mathrm{T}}\\\label{eq:so_start}
    &S_2=P_{10,10} - P_{10,10}^{\mathrm{T}}\\
    &S_3=P_{100,100} - P_{100,100}^{\mathrm{T}}\\
    &S_4=\mathrm{Re}[F[P_{1,1}- P_{1,1}^{\mathrm{T}}]F^{-1}]\\
    &S_5=\mathrm{Re}[F[P_{10,10}- P_{10,10}^{\mathrm{T}}]F^{-1}]\\
    &S_6=\mathrm{Re}[F[P_{100,100}- P_{100,100}^{\mathrm{T}}]F^{-1}],\label{eq:so_end}
\end{align}
where $P_{m,n}$ is the shift operator that rolls the input image for $m$ places along the height coordinate and rolls the input image for $n$ places along the width coordinate.
The sampling results are shown in Fig.~\ref{fig:ffhq_s_ms}, where we set $\omega = 1000$. It is observed that again all these solenoidal terms do not impose an obvious effect on the sampling quality. Additionally, as displayed in Fig.~5 in the supplementary material, these solenoidal terms also do not {have} an obvious effect on acceleration. We leave {a} more extensive study of the solenoidal terms in the further work.

\begin{figure}[h]
    \centering
        \includegraphics[width=0.48\textwidth]{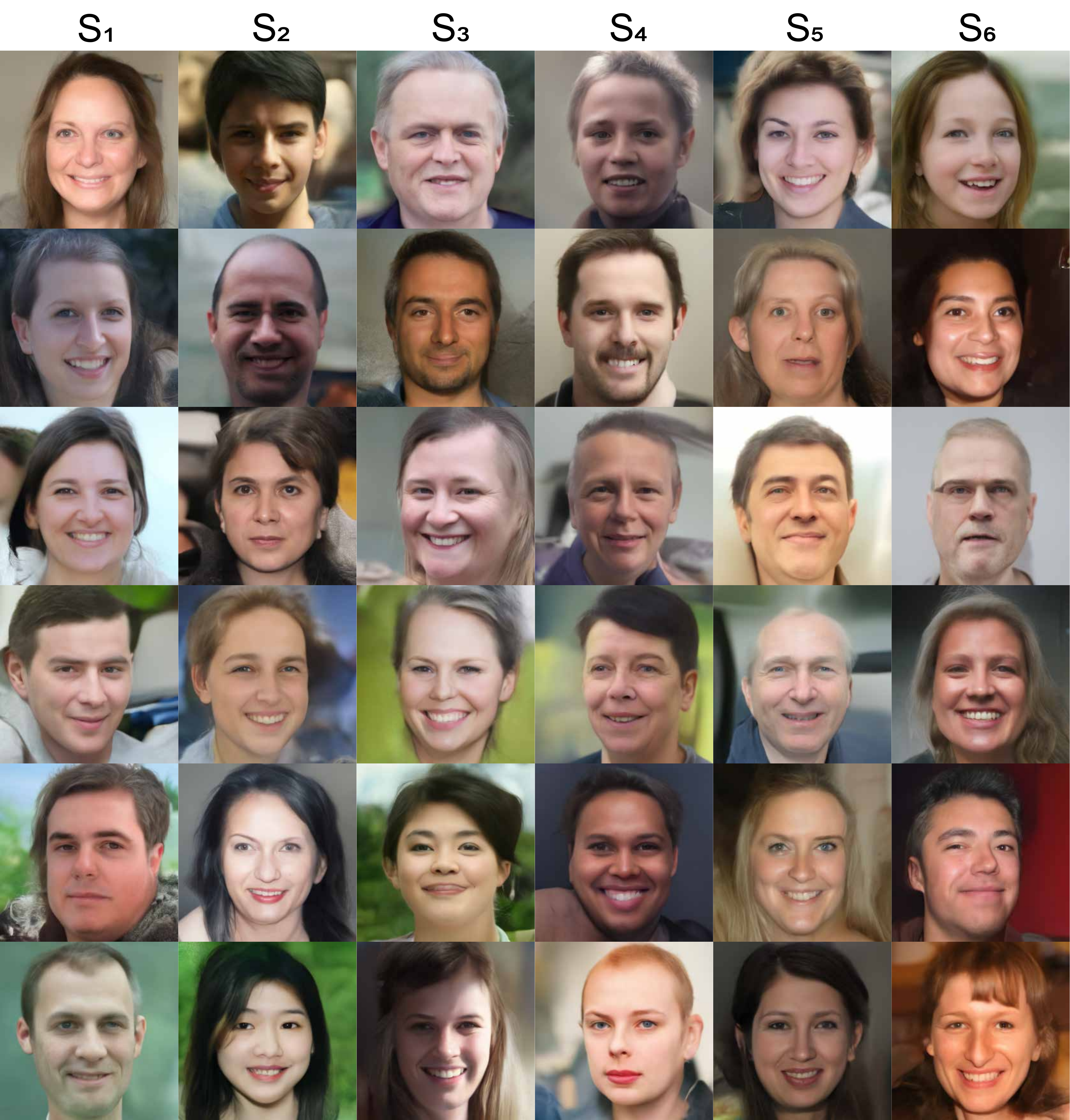}
        \caption{
        Facial images generated by NCSN++~\citep{song2020score}
        with our PDS using different solenoidal items.
        Iterations $T$: $66$.
        Dataset: FFHQ ($1024\times 1024$)~\citep{karras2019style}.
    		}
        \label{fig:ffhq_s_ms}
\end{figure}

%% file: sections/9_conclusion.tex
\section{Conclusion}
We have proposed a novel preconditioned diffusion sampling (PDS) method for accelerating off-the-shelf score-based generative models (SGMs), without model retraining. Considering the sampling process as a combination of Langevin dynamics and reverse diffusion, we reveal that existing sampling suffers from the ill-conditioned issue. To solve this, we reformulate both the Langevin dynamics and the reverse diffusion with {statistics-based} matrix preconditioning. Theoretically, PDS preserves the steady-state (\ie, target) distribution of the Langevin dynamics and the final-state distribution of the reverse diffusion respectively. We show that PDS significantly accelerates existing state-of-the-art SGMs while maintaining the generation quality, matching or even surpassing representative SGMs and DDPMs.

%% file: Appendix.tex
\begin{appendices}

\section{Implementations}
For CIFAR-10, LSUN (bedroom and church, $256\times256$), and FFHQ ($1024\times1024$), we apply checkpoints of NCSN++ provided in~\citep{song2020score}. We train NCSN++ on CelebA ($64\times 64$) for 600k iterations at the batch size of 64 on two NVIDIA RTX 3090 GPUs. We apply the same hyper-parameters of optimizer and learning schedule as~\citep{song2020score}. For ImageNet and COCO datasets, we train the NCSN++~\citep{song2020score} by ourselves at the resolution of 64$\times$64 as the baseline. We apply the same hyper-parameters of optimizer and learning schedule as~\citep{song2020score}. We train the model for 2400k iterations for ImageNet and 600k for COCO. The batch size is set at 256 for both datasets. For sampling, we use the same generation process of continuous variance exploding (VE)-NCSN++~\citep{song2020score} with both predictor and corrector and all the experiments apply the linear schedule used in~\citep{song2020score} with both predictor and corrector (PC).

\section{Experiments on latent space}
{We assess the performance of our PDS on SGMs in latent spaces. Since existing latent SGMs predominantly rely on VP SDE or ODE frameworks for generation, which are incompatible with PDS, we train custom latent SGMs. For the MNIST dataset, we first train a convolutional variational autoencoder (VAE) to map images into a 384-dimensional latent space, then train an NCSN on the VAE-transformed dataset. Similarly, for the CelebA ($64 \times 64$) dataset, we use a convolutional VAE to transform  images into a 512-dimensional latent space and train an NCSN on this transformed dataset. We apply PDS with $\alpha=1.75$ for MNIST and $\alpha=4$ for CelebA. As shown in Fig.~\ref{fig:latent_mnist} and Fig.~\ref{fig:latent_celeba}, PDS substantially preserves sampling quality, while achieving \(20\times\) and \(10\times\) acceleration for the two datasets, respectively.}

The latent SGM serves as a nonlinear counterpart to the proposed PDS. While latent SGMs rely on a nonlinear VAE for sampling space transformation, PDS employs a linear operator to achieve similar effects. Both methods address ill-conditioning, but they differ in execution: latent SGMs utilize the dimensionality reduction of a VAE, whereas PDS directly applies an efficient linear operator. Latent SGMs, however, require training a VAE and subsequently retraining the SGM for acceleration, whereas PDS bypasses this step by leveraging dataset statistics, enabling seamless integration with existing SGMs without retraining or additional VAE training.
Furthermore, experimental results (Fig.~\ref{fig:latent_mnist} and Fig.~\ref{fig:latent_celeba}) indicate that PDS enhances the sampling quality of latent SGMs. This suggests that the VAE does not fully resolve the ill-conditioning inherent in the latent SGM sampling process.

\begin{figure}[h]
    \centering
    \includegraphics[width=0.48\textwidth]{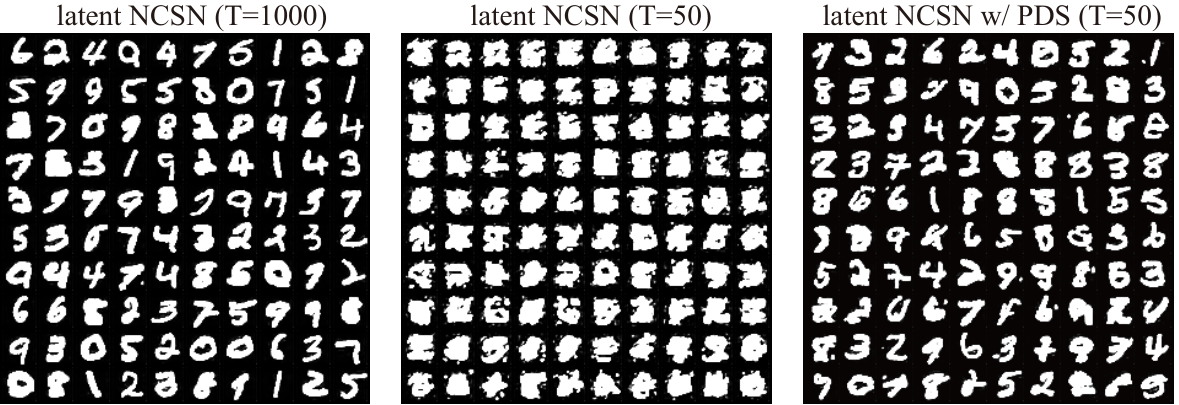}
        \caption{
        Images generated by NCSN~\citep{song2020score} on the latent MNIST with or without our PDS under different iterations $T$.
        }
        \label{fig:latent_mnist}
\end{figure}

\begin{figure}[h]
    \centering
    \includegraphics[width=0.48\textwidth]{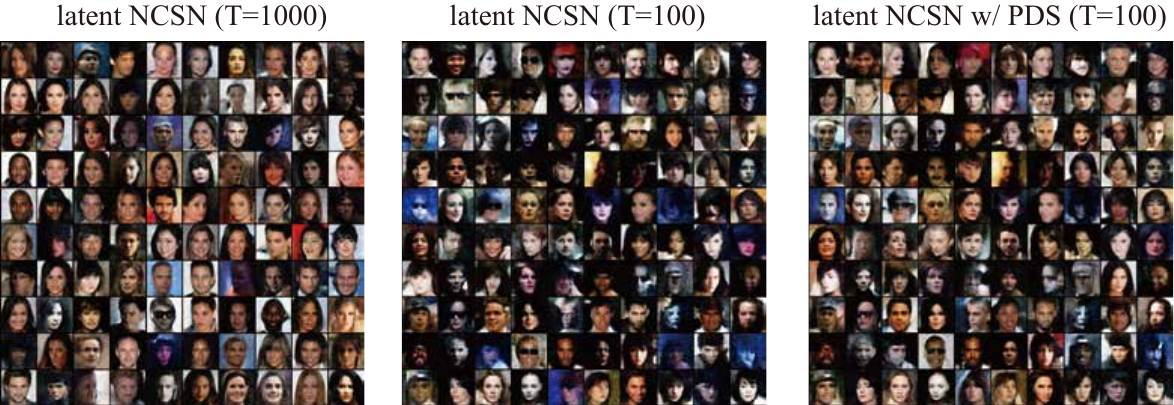}
        \caption{
        Images generated by NCSN~\citep{song2020score} on the latent CelebA ($64\time 64$) with or without our PDS under different iterations $T$.
        }
        \label{fig:latent_celeba}
\end{figure}

\section{More samples}
We provide more samples generated by NCSN++ with PDS in Fig.~\ref{fig:ffhq_s_ms_diff}-\ref{fig: ncsnpp_lsun}.

\begin{figure}[h]
    \centering
    \includegraphics[width=0.43\textwidth]{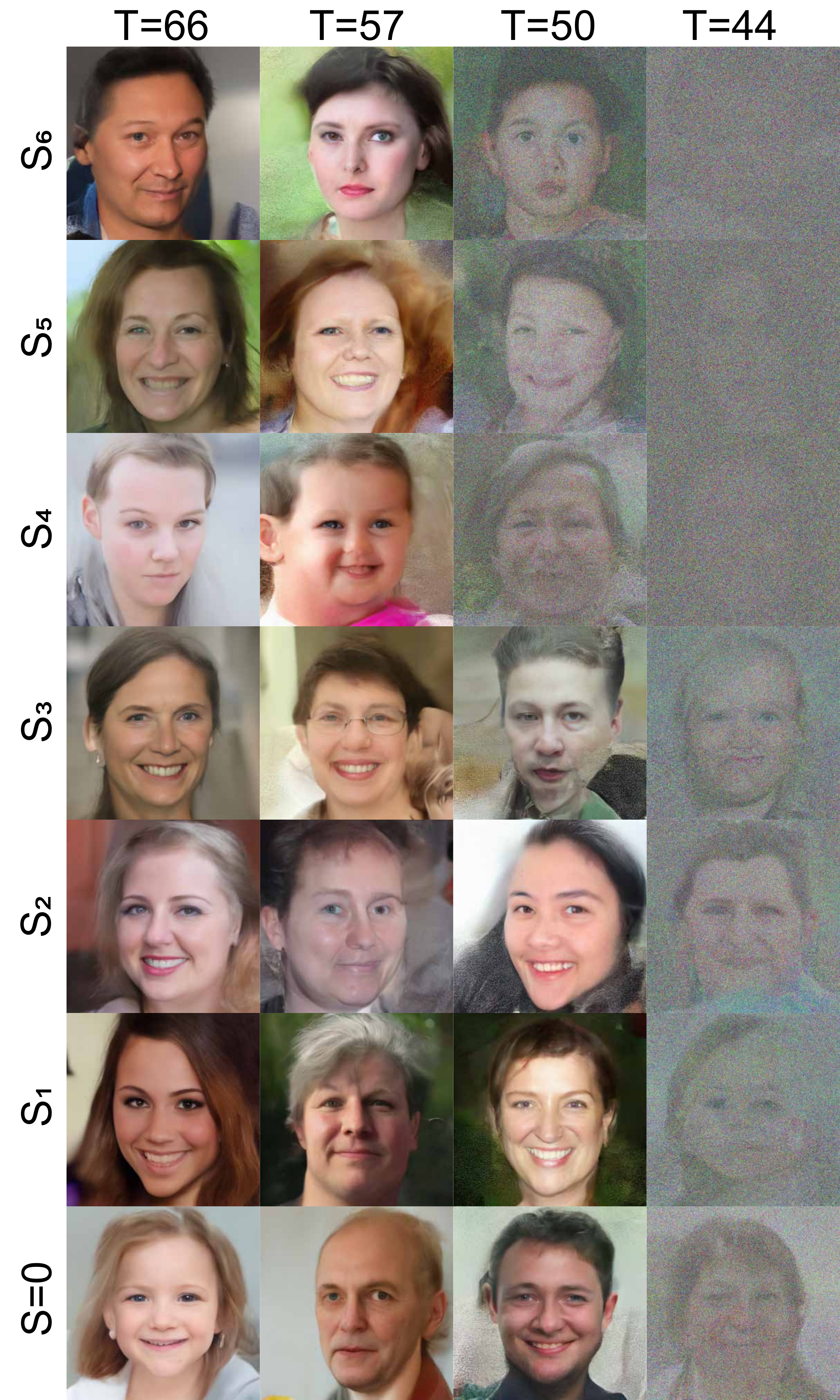}
        \caption{
        Facial images generated by NCSN++~\citep{song2020score} with our PDS under different iterations $T$ and different solenoidal items described in Eq.~(55)-(59). Dataset: FFHQ~\citep{karras2019style} ($1024\times 1024$).
        }
        \label{fig:ffhq_s_ms_diff}
\end{figure}

\begin{figure}[h]
    \centering
    \includegraphics[width=0.47\textwidth]{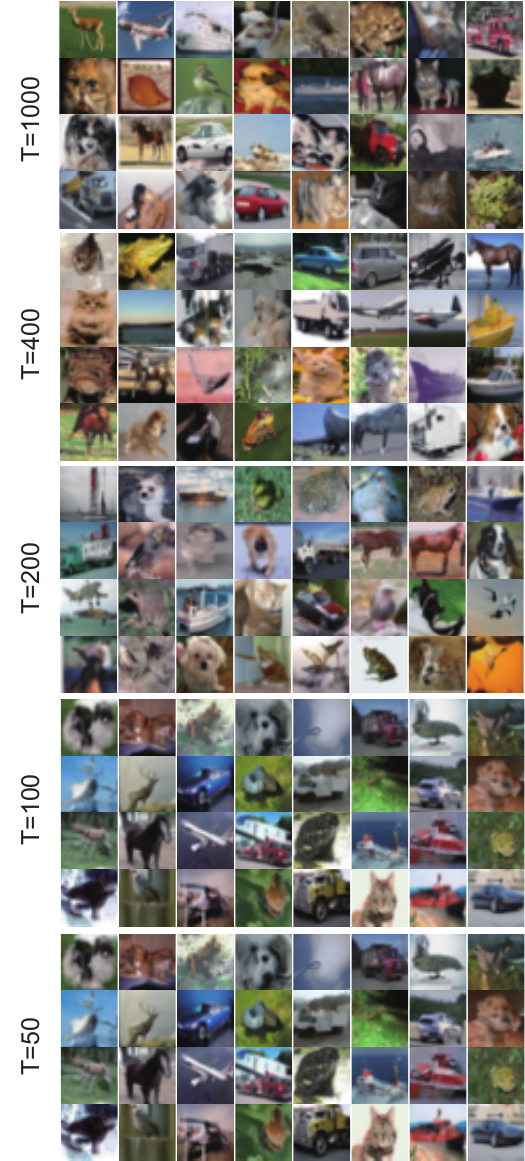}
    \caption{CIFAR-10 samples generated by NCSN++~\citep{song2020score} with PDS under different iterations $T$.}
    \label{fig:cifar10_a}
\end{figure}

\begin{figure}[h]
    \centering
\includegraphics[width=0.45\textwidth]{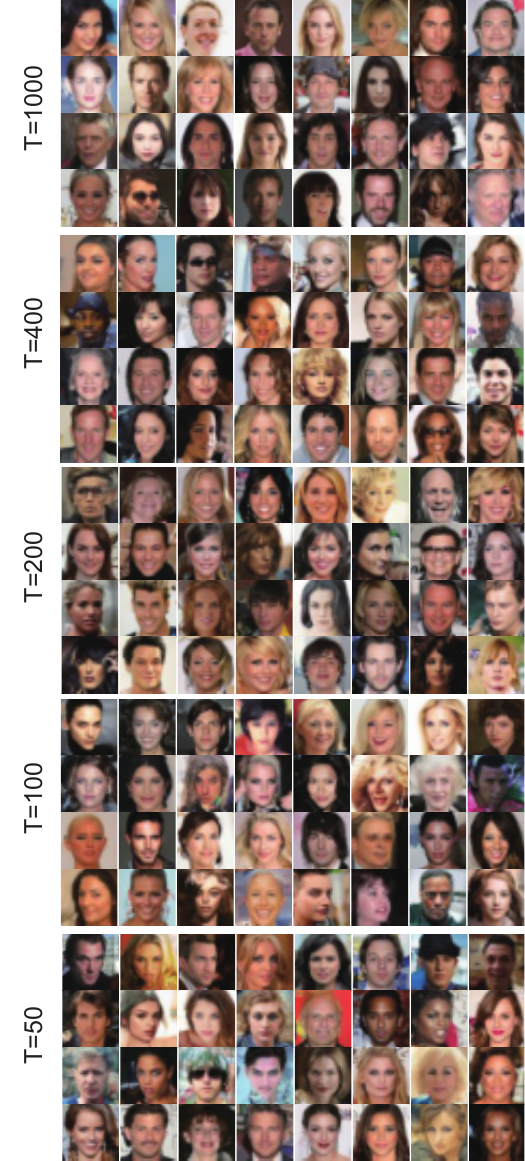}
        \caption{CelebA ($64\times64$) samples generated by NCSN++~\citep{song2020score} with PDS under different iterations $T$.}
    \label{fig:celeba_a}
\end{figure}

\begin{figure*}[h]
    \centering
    \includegraphics[width=\textwidth]{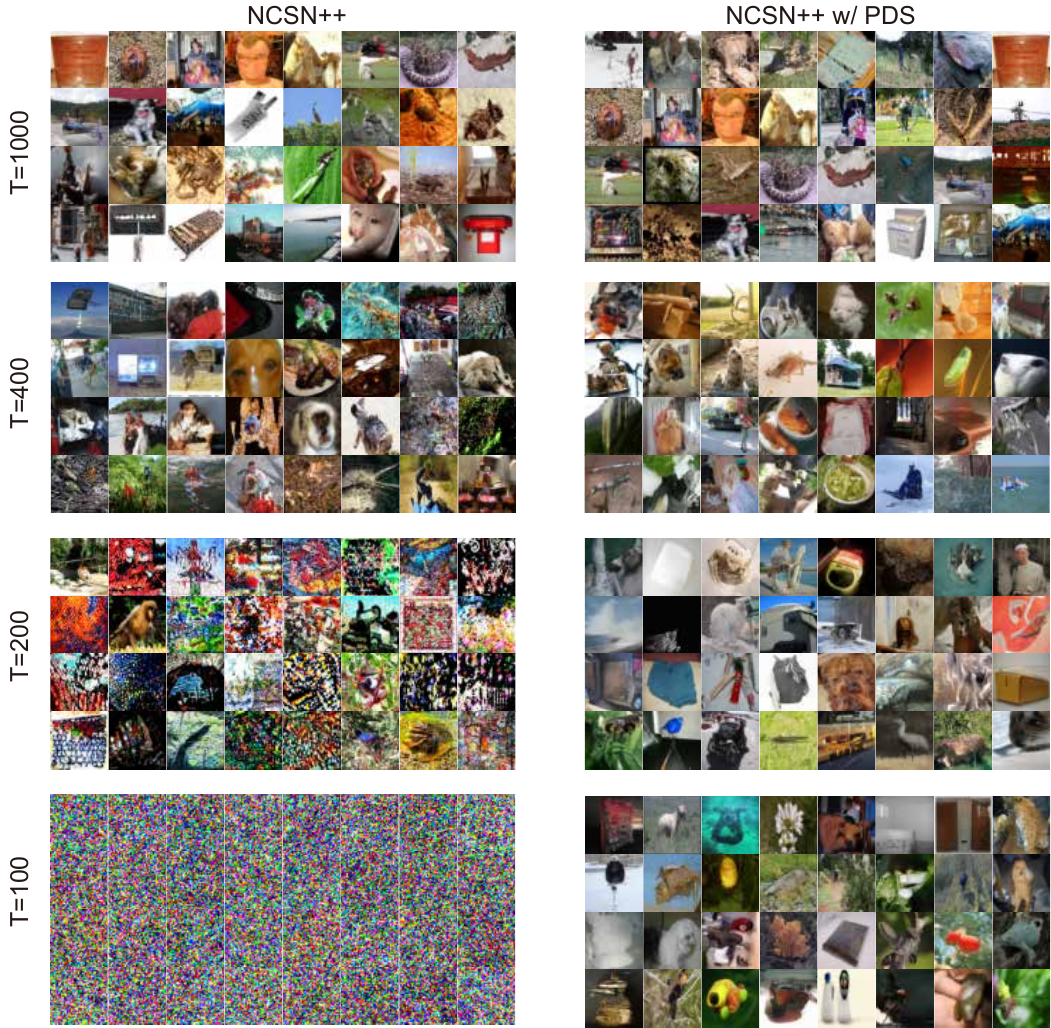}
        \caption{ImageNet ($64\times64$) samples generated by NCSN++~\citep{song2020score} with PDS under different iterations $T$.}
    \label{fig:imagenet}
\end{figure*}

\begin{figure*}[h]
    \centering
    \includegraphics[width=\textwidth]{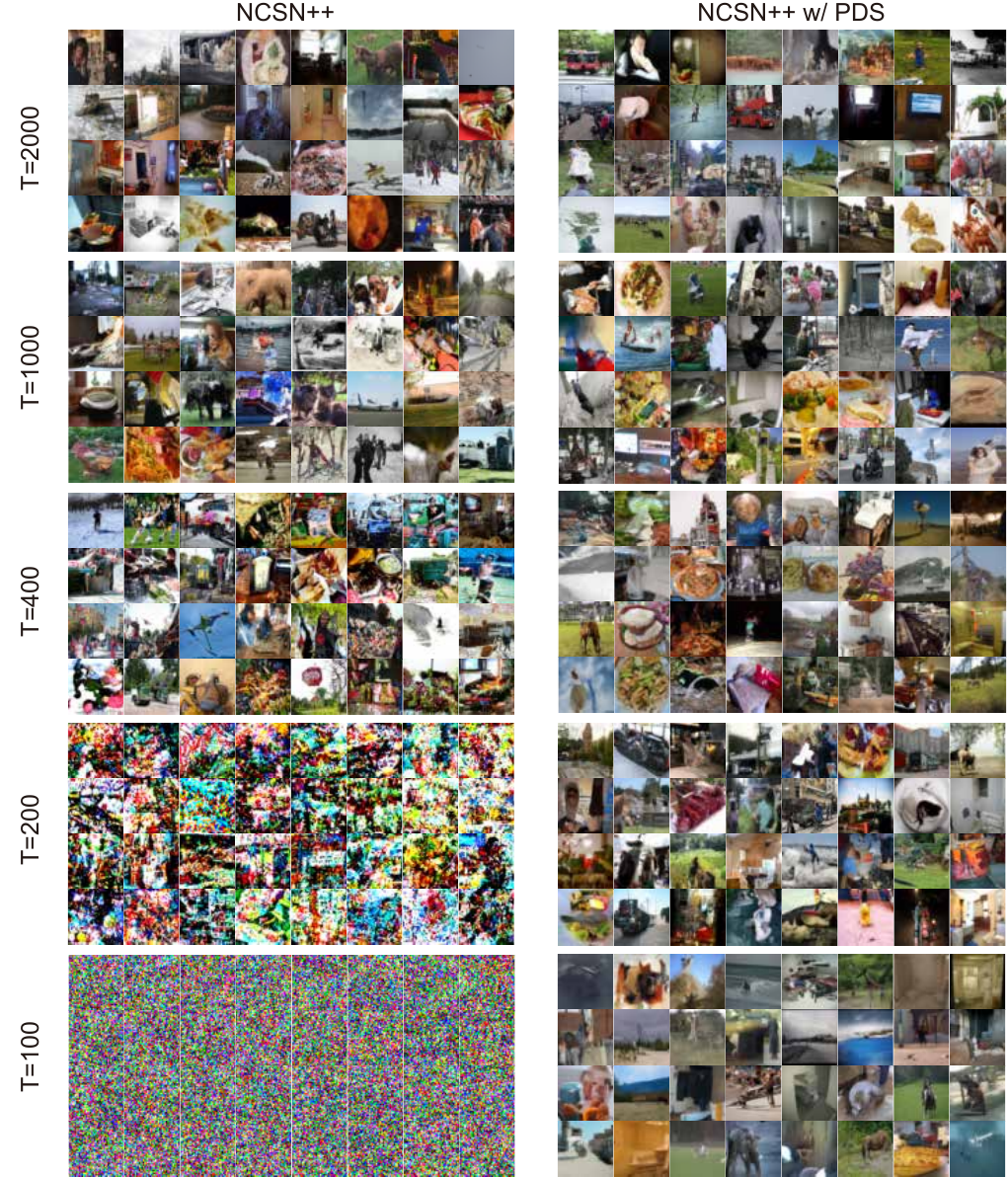}
        \caption{COCO ($64\times64$) samples generated by NCSN++~\citep{song2020score} with PDS under different iterations $T$.}
    \label{fig:coco}
\end{figure*}

\begin{figure}[h]
    \centering
    \centerline{\includegraphics[width=0.48\textwidth]{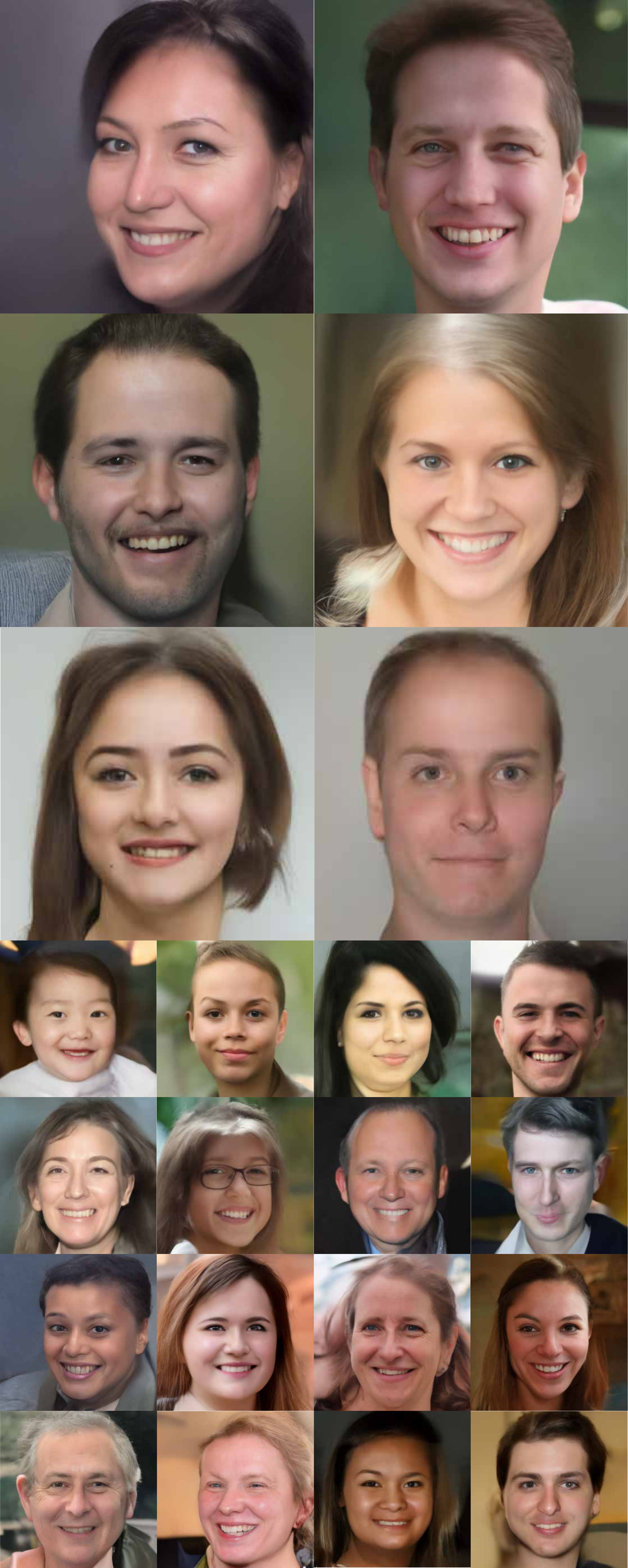}}
    \caption{
    Facial images at a resolution of $1024\times 1024$ generated by NCSN++~\citep{song2020score}
    with our PDS. Sampling iterations: $66$. Dataset: FFHQ~\citep{karras2019style}.
    }
    \label{fig:ffhq_hq}
\end{figure}

\begin{figure}[h]
        \centerline{\includegraphics[width=0.48\textwidth]{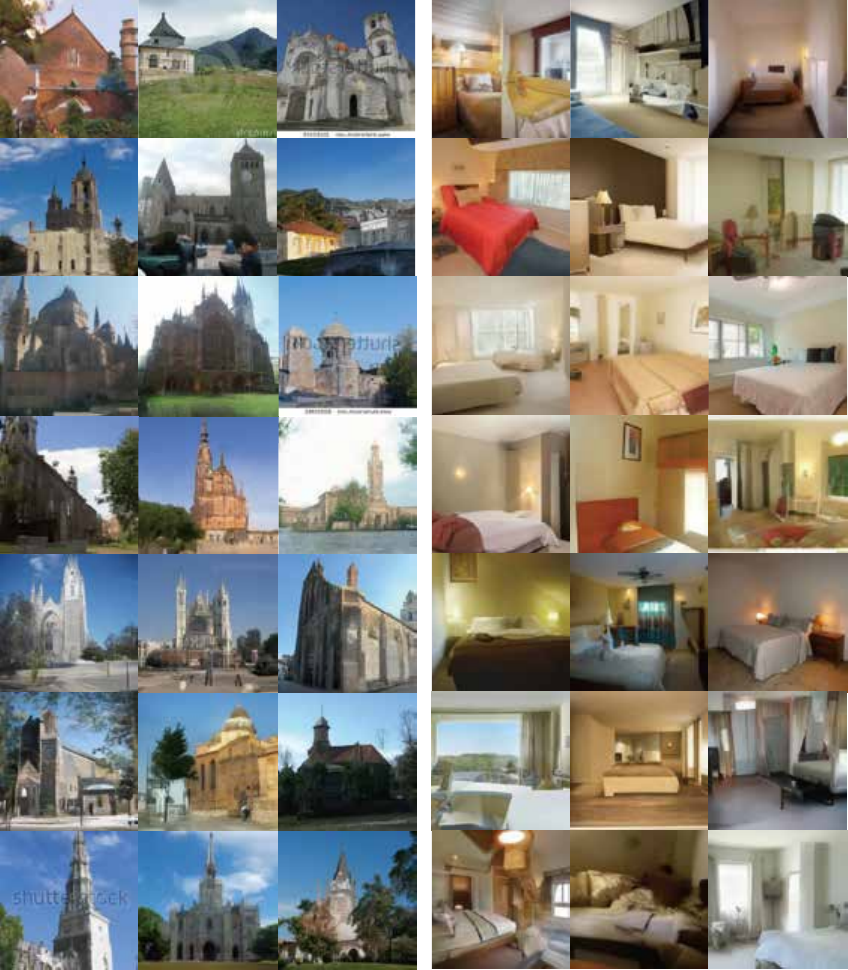}}
       \caption{Sampling using NCSN++~\citep{song2020score} with PDS on LSUN (church and bedroom) at a resolution of $256 \times 256$ under iteration numbers $80$.}
    \label{fig: ncsnpp_lsun}
\end{figure}

\end{appendices}